\pdfoutput=1

\documentclass[11pt]{article}

\usepackage[preprint]{acl}

\usepackage{times}
\usepackage{latexsym}

\usepackage[T1]{fontenc}

\usepackage[utf8]{inputenc}

\usepackage{microtype}

\usepackage{inconsolata}

\def\a{{\bf a}}

\def\c{{\bf c}}

\def\e{{\bf e}}

\def\S{{\bf S}}

\def\x{{\bf x}}
\def\y{{\bf y}}
\def\z{{\bf z}}

\def\w{{\bf w}}
\def\0{{\bf 0}}
\def\1{{\bf 1}}

\def\NM{{\mathcal N}}

\def\LM{{\mathcal L}}

\def\DM{{\mathcal D}}

\def\RB{{\mathbb R}}
\def\EB{{\mathbb E}}

\def\alp{\mbox{\boldmath$\alpha$\unboldmath}}

\def\Om{\mbox{\boldmath$\Omega$\unboldmath}}

\def\ph{\mbox{\boldmath$\phi$\unboldmath}}

\def\pii{\mbox{\boldmath$\pi$\unboldmath}}

\def\tha{\mbox{\boldmath$\theta$\unboldmath}}
\def\Tha{\mbox{\boldmath$\Theta$\unboldmath}}
\def\muu{\mbox{\boldmath$\mu$\unboldmath}}
\def\Si{\mbox{\boldmath$\Sigma$\unboldmath}}

\def\Gam{\mbox{\boldmath$\Gamma$\unboldmath}}
\def\gamm{\mbox{\boldmath$\gamma$\unboldmath}}
\def\Lam{\mbox{\boldmath$\Lambda$\unboldmath}}
\def\lamm{\mbox{\boldmath$\lambda$\unboldmath}}

\newcommand{\tabref}[1]{Table~\ref{#1}}
\newcommand{\secref}[1]{Sec.~\ref{#1}}
\newcommand{\figref}[1]{Fig.~\ref{#1}}
\newcommand{\lemref}[1]{Lemma~\ref{#1}}
\newcommand{\thmref}[1]{Theorem~\ref{#1}}

\newcommand{\defref}[1]{Definition~\ref{#1}}
\newcommand{\eqnref}[1]{Eq.~\ref{#1}}
\newcommand{\algref}[1]{Alg.~\ref{#1}}
\newcommand{\appref}[1]{Appendix~\ref{#1}}

\renewcommand{\tilde}{\widetilde}
\renewcommand{\hat}{\widehat}
\renewcommand{\frac}{\tfrac}

\definecolor{green}{rgb}{0,0.5,0}

\def\blue#1{\textcolor{blue}{#1}}

\def\red#1{\textcolor{red}{#1}}
\def\green#1{\textcolor{green}{#1}}
\def\purple#1{\textcolor{purple}{#1}}
\def\brown#1{\textcolor{brown}{#1}}
\def\orange#1{\textcolor{orange}{#1}}
\usepackage{amsmath}
\usepackage{amssymb}
\usepackage{mathtools}
\usepackage{amsthm}
\usepackage{wrapfig}
\usepackage{booktabs}
\usepackage[ruled,vlined,noline]{algorithm2e}
\usepackage{algorithmicx}
\usepackage{algpseudocode}
\usepackage{enumitem}
\renewcommand{\algref}[1]{Alg.~\ref{#1}}
\usepackage{multirow}
\usepackage[nice]{nicefrac}

\usepackage[capitalize,noabbrev]{cleveref}
\theoremstyle{plain}
\newtheorem{theorem}{Theorem}[section]

\newtheorem{lemma}[theorem]{Lemma}

\theoremstyle{definition}
\newtheorem{definition}{Definition}[section]

\theoremstyle{remark}

\title{Variational Language Concepts for Interpreting \\Foundation Language Models}

\author{Hengyi Wang\thanks{Correspondence to: Hengyi Wang \textless hengyi.wang@rutgers.edu\textgreater}~~~~Shiwei Tan~~~~Zhqing Hong~~~~Desheng Zhang~~~~Hao Wang
\\\\Department of Computer Science, Rutgers University
}

\begin{document}
\maketitle
\begin{abstract}
Foundation Language Models (FLMs) such as BERT and its variants have achieved remarkable success in natural language processing. To date, the interpretability of FLMs has primarily relied on the attention weights in their self-attention layers. However, these attention weights only provide word-level interpretations, failing to capture higher-level structures, and are therefore lacking in readability and intuitiveness. 
To address this challenge, we first provide a formal definition of \emph{conceptual interpretation} and then 
propose a variational Bayesian framework, dubbed VAriational Language Concept (VALC), to go beyond word-level interpretations and provide concept-level interpretations. 
Our theoretical analysis shows that our VALC finds the optimal language concepts to interpret FLM predictions. Empirical results on several real-world datasets show that our method can successfully provide conceptual interpretation for FLMs\footnote{{Code will soon be available at https://github.com/Wang-ML-Lab/interpretable-foundation-models}}. 
\end{abstract}

\section{Introduction} \label{sec:intro}
Foundation language models (FLMs) such as BERT~\citep{devlin2018bert} and its variants~\citep{lan2019albert,liu2019roberta,he2021deberta,portes2023mosaicbert} have achieved remarkable success in natural language processing. These FLMs are usually large attention-based neural networks that follow a pretrain-finetune paradigm, where models are first pretrained on large datasets and then finetuned for a specific task. As with any machine learning models, interpretability in FLMs has always been a desideratum, especially in decision-critical applications (e.g., healthcare). 


\begin{figure}[t]
        \centering
        \includegraphics[width = 0.48\textwidth]{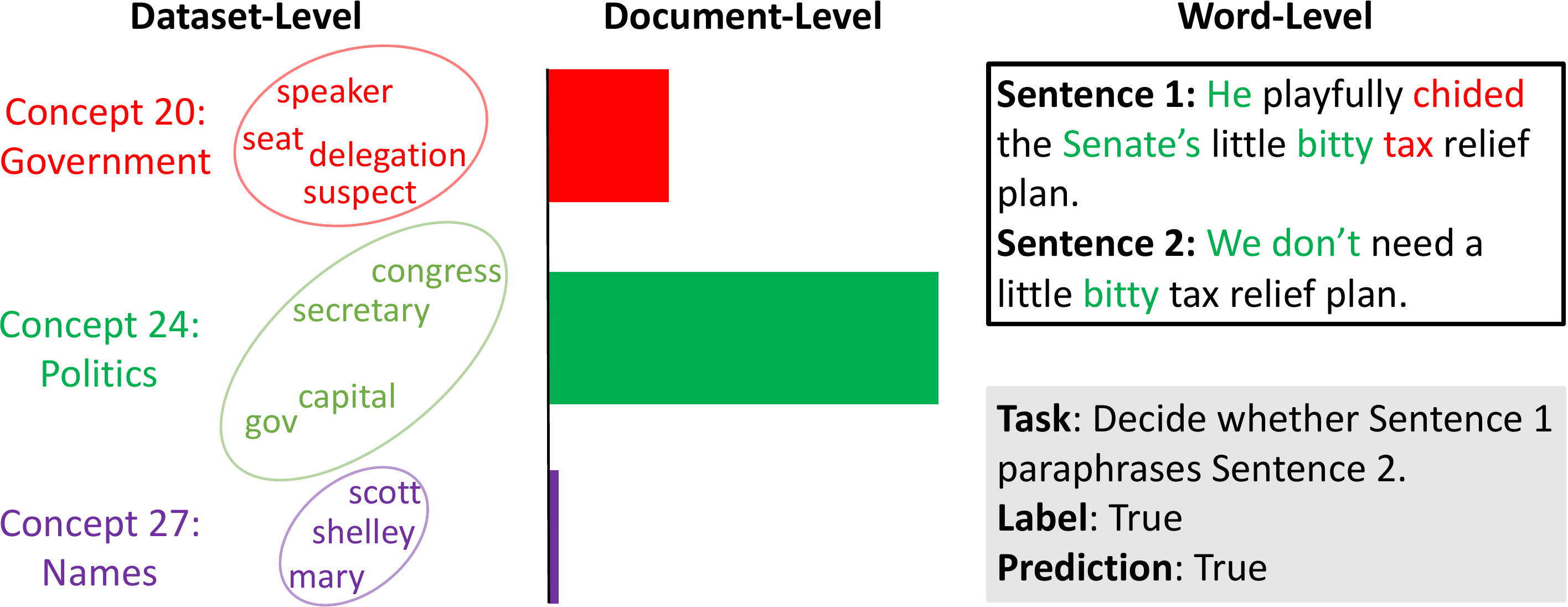}
        \caption[width = 0.8\linewidth]{Visualization of VALC's learned concepts. A document consists of two sentences. The task is to decide whether `Sentence 1' paraphrases `Sentence 2'.  
        \textbf{Left:} Dataset-level concepts for MRPC dataset with $3$ concepts and their nearest word embeddings. \textbf{Middle:} Document-level concept strength, showing that this document is mostly related to \red{Concept 20} and \green{Concept 24}. 
        \textbf{Right:} Word-level concepts, where the FLM correctly predicts the label to be `True', and VALC interprets that this is because the both sentences consist of words with \green{Concept 24}, i.e., \emph{Politics}.}
        \label{fig:teaser}
        \vskip -0.5cm
\end{figure}
To date, FLMs' interpretability 
has primarily relied on the attention weights in self-attention layers. However, these attention weights only provide raw word-level importance scores as interpretations. Such low-level interpretations fail to capture higher-level semantic structures, and hence lack {readability and intuitiveness}. For example, low-level interpretations often fail to capture influence of similar words to predictions, leading to unstable or even unreasonable explanations (see Sec.~\ref{sec:interpret} for details). 

In this paper, we aim to go beyond word-level attention and interpret FLM predictions at the concept level. Such higher-level semantic interpretations are complementary to word-level importance scores and often more readable and intuitive. 
{For example, as shown in~\figref{fig:teaser}, VALC interprets the FLM with the following multi-level concepts (details in~\appref{app:interpret_more}):}
\begin{itemize}[nosep]
    \item \textbf{Dataset-level} concepts are highlighted by the top words and the distribution of their embeddings in the PLM (Fig. 1(left)). For example, \emph{Concept 20 (Government)} corresponds to the red ellipse, encompassing words relevant to government entities and activities, as shown in Fig. 1(left).
  \item \textbf{Document-level} concepts are demonstrated by each document's topics; for instance, in the 3 bars representing probability distribution over 3 concepts for the document in Fig. 1(middle), VALC identifies \emph{Concept 24}, i.e., `politics', and \emph{Concept 20}, i.e., `government', as considerably more relevant concepts compared to \emph{Concept 27}, i.e. `names'.
  \item \textbf{Word-level} concepts are identified by words in documents. For example, in the box displaying the document in Fig. 1(right), VALC highlights the words `chided' and `tax' because they are highly related to \emph{Concept 20}, i.e., `government'. Terms like `Senate' and `bitty' are associated with \emph{Concept 24}, i.e. `politics', aligning with the document-level concepts.
\end{itemize}

We start by developing a comprehensive and formal definition of \emph{conceptual interpretation} with four desirable properties: (1) multi-level structure, (2) normalization, (3) additivity, and (4) mutual information maximization. With this definition, we then propose a variational Bayesian framework, dubbed VAriational Language Concept (VALC), to provide \emph{dataset-level}, \emph{document-level}, and \emph{word-level} (the first property) conceptual interpretation for FLM predictions. 
Our theoretical analysis shows that maximizing our VALC's evidence lower bound is equivalent to inferring the optimal conceptual interpretation with \emph{Properties (1-3)} while maximizing the mutual information between the inferred concepts and the observed embeddings from FLMs, i.e., \emph{Property (4)}.

Drawing inspiration from hierarchical Bayesian deep learning~\cite{BDL,BDLSurvey,NPN}, the core of our idea is to treat a FLM's contextual word embeddings (and their corresponding attention weights) as observed variables and build a probabilistic generative model to automatically infer the higher-level semantic structures (e.g., concepts or topics) from these embeddings and attention weights, thereby interpreting the FLM's predictions at the concept level. 
{Our VALC is compatible with any attention-based FLMs and can work as an conceptual interpreter, which explains the FLM predictions at  multiple levels with theoretical guarantees. }
Our contributions are as follows:
\begin{itemize}[nosep]
\item We identify the problem of multi-level interpretations for FLM predictions, develop a formal definition of \emph{conceptual interpretation}, and propose VALC as the first general method to infer such conceptual interpretation. 
\item {Theoretical analysis shows that learning VALC is equivalent to inferring the optimal conceptual interpretation according to our definition.} 
\item Quantitative and qualitative analysis on real-world datasets show that VALC can infer meaningful language concepts to effectively and intuitively interpret FLM predictions. 
\end{itemize}

\section{Related Work}
\textbf{Foundation Language Models.} 
Foundation language models are large attention-based neural networks that follow a pretrain-finetune paradigm. Usually they are first pretrained on large datasets in a self-supervised manner and then finetuned for a specific downstream task. 
BERT~\citep{devlin2018bert} is a pioneering FLM that has shown impressive performance across multiple downstream tasks. 
{Following BERT, there have been 
variants~\cite{he2021deberta,clark2020electra,yang2019xlnet,liu2019roberta,lewis2019bart} that design different self-supervised learning objectives or training schemes to achieve better performance.} 
{While FLMs offer attention weights for interpreting predictions at the word level, these interpretations lack readability and intuitiveness because they fail to capture higher-level semantic structures.}


\textbf{Interpretation Methods for FLMs.} 
Existing conceptual interpretation methods for FLMs typically rely on topic models~\cite{blei2003latent,blei2006dynamic,blei2012probabilistic,wang2012continuous,chang2009relational,mcauliffe2007supervised,hoffman2010online} and prototypical part networks~\cite{chen2019looks}. There has been recent work that employs deep neural networks to learn topic models more efficiently~\citep{card2017neural,xing2017topic,peinelt2020tbert}, using techniques such as amortized variational inference. There is also work that improves upon traditional topic models by either leveraging word similarity as a regularizer for topic-word distributions~\citep{das2015gaussian,batmanghelich2016nonparametric} or including word embeddings into the generative process~\citep{hu2012latent,dieng2020topic,bunk2018welda,duan2021sawtooth}.
There is also work that builds topic models upon embeddings from FLMs~\citep{grootendorst2020bertopic, zhang2022cetoic,wang2022mixtures,zhao2020optimal,meng2022topic}. However, these methods often rely on a pipeline involving dimensionality reduction and basic clustering, which is not end-to-end, leading to potential information loss between FLM embeddings and clustering outcomes. This can result in \emph{unfaithful} interpretations for the underlying FLM. Additionally, they typically generate interpretations at a single level (e.g., document level), lacking a multi-level conceptual structure. 

Beyond topic models, attribution-based approaches such as LIME~\cite{ribeiro2016lime} and SHAP~\cite{lundberg2017shap} assign importance to input features to explain predictions.  {Concept bottleneck models (CBMs)~\cite{koh2020CBM, yuksekgonul2023posthoc, yang2023language,kim2018tcav,schulz2020restricting,paranjape2020information,schrouff2021best} offer interpretations by learning conceptual activation and then performing classifications on these concepts, while inherent models~\cite{xie2023proto,ren2023defining,shi2021corpus} focus on model redesign/re-training for interpretability. However, these approaches often require extra supervision or re-training, making them unsuitable for our setting.
In contrast, our method is inherently multi-level and end-to-end, models {\emph{concepts}} across dataset, document, and word levels, and produces faithful post-hoc interpretations for any models based on {FLMs} with theoretical guarantees.}

\section{Methods}

In this section, we formalize the definition of \emph{conceptual interpretation}, and describe our proposed VALC for conceptual interpretation of FLMs.


\subsection{Problem Setting and Notation}  \label{sec:notation}
We consider a corpus of $M$ documents, where the $m$'th document contains $J_m$ words, and a FLM $f(\mathcal{D}_m)$, which takes as input the document $m$ (denoted as $\mathcal{D}_m$) with $J_m$ words and outputs (1) a CLS embedding $\c_m\in\mathbb{R}^d$, (2) $J_m$ contextual word embeddings $\e_m \triangleq [\e_{mj}]_{j=1}^{J_m}$, and (3) 
the attention weights $\a_m^{(h)} \triangleq [a_{mj}^{(h)}]_{j=1}^{J_m}$ between each word and the last-layer CLS token, where $h$ denotes the $h$'th attention head. We denote the average attention weight over H heads as $a_{mj}=\frac{1}{H}\sum_{h=1}^H a_{mj}^{(h)}$ and correspondingly $\a_m \triangleq [a_{mj}]_{j=1}^{J_m}$ (see the FLM at the bottom of \figref{fig:overview}). 
\begin{figure}[t]
        \centering
        \includegraphics[width = 0.40\textwidth]
        {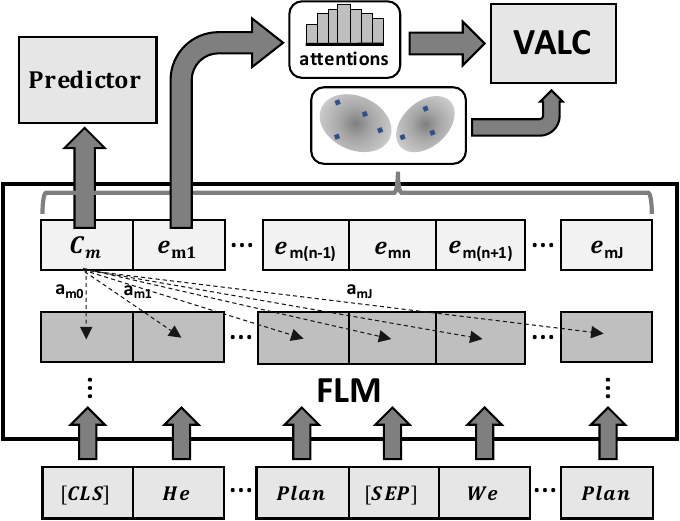}
        \caption[width = 0.8\linewidth]{Overview of VALC framework. }
        \label{fig:overview}
\end{figure}
In FLMs, these last-layer CLS embeddings are used as document-level representations for downstream tasks (e.g., document classification). 
Furthermore, our VALC assumes $K$ concepts (topics) for the corpus. For document $m$, our VALC interpreter tries to infer a concept distribution vector $\tha_m\in \mathbb{R}^K$ (also known as the topic proportion in topic models) for the whole document and a concept distribution vector $\ph_{mj}=[\phi_{mjk}]_{k=1}^K \in \mathbb{R}^K$ for word $j$ in document $m$. In our continuous embedding space, the $k$'th concept is represented by a Gaussian distribution, $\mathcal{N}(\muu_k,\Si_k)$, of contextual word embeddings; we use shorthand $\Om_k=(\muu_k,\Si_k)$ for brevity. 
The goal is to interpret FLMs' predictions \emph{at the concept level} using the inferred document-level concept vector $\tha_m$, word-level concept vector $\ph_{mj}$, and the learned embedding distributions $\{\mathcal{N}(\muu_k,\Si_k)\}_{k=1}^K$ for each concept (see~\secref{sec:interpret} for detailed descriptions and visualizations). 

\subsection{Formal Definition of Language Concepts}\label{sec:def}

Below we formally define `conceptual interpretation' for FLM predictions (see notations in~\secref{sec:notation}):

\begin{definition}[\textbf{Conceptual Interpretation}]\label{def:concept} 
Assume $K$ concepts and a dataset $\DM$ containing $M$ documents, each with $J_m$ words ($1\le m\le M$). Conceptual interpretation for a document $m$ consists of $K$ \emph{dataset-level} variables $\{\Om_k\}_{k=1}^K$, a \emph{document-level} variable $\tha_m$, and $J_m$ \emph{word-level} variables $\{\ph_{mj}\}_{j=1}^{J_m}$ with the following properties:
\begin{enumerate}[nosep,leftmargin=18pt]
    \item[(1)] \textbf{Multi-Level Structure.} 
    Conceptual interpretation has a three-level structure: 
    \begin{enumerate}[nosep,leftmargin=10pt]
    \item[(a)] Each \emph{dataset-level} variable $\Om_k=(\muu_k,\Si_k)$ describes the $k$'th concept; $\muu_k\in\RB^{d}$ and $\Si_k\in\RB^{d\times d}$ denote the mean and covariance of the $k$'th concept in the embedding space (i.e., $\e_{mj}\in\RB^d$), respectively. 
    \item[(b)] Each \emph{document-level} variable $\tha_m\in\RB^{K}_{\geq 0}$ describes document $m$'s relation to the $K$ concepts. 
    \item[(c)] Each \emph{word-level} variable $\ph_{mj}\in \RB^{K}_{\geq 0}$ describes word $j$'s relation to the $K$ concepts.     
    \end{enumerate}
    \item[(2)] \textbf{Normalization.}~The document- and word-level interpretations, $\tha_m$ and $\ph_{mj}$, are normalized:
    \begin{enumerate}[nosep,leftmargin=10pt]
        \item[(a)] $\sum_{k=1}^{K} \theta_{mk} = 1$ for document $m$.
        \item[(b)] $\sum_{k=1}^{K} \phi_{mjk} = 1$ for word $j$ in document $m$.
    \end{enumerate}
    \item[(3)] \textbf{Additivity.}~We can add/subtract the $k$'s concept from the contextual embeddings $\e_{mj}$ of word $j$ in document $m$, i.e., $\e_{mj} \leftarrow \e_{mj} \pm x_k \muu_k $ {($x_k$ is the editing weight of concept $k$)}. 
    \item[(4)] \textbf{Mutual Information Maximization.} 
    The conceptual interpretation achieves maximum mutual information between the {observed embeddings} $\e_m$ in FLMs and the document-level/word-level interpretation, $\tha_m$ and $\ph_{mj}$. 
\end{enumerate}
\end{definition}

    
In~\defref{def:concept}, Property (1) provides comprehensive three-level conceptual interpretation for FLM predictions, Property (2) ensures proper normalization in concept assignment at the document and word levels, Property (3) enables better concept editing (more details in~\secref{sec:exp_prune}) to modify FLM predictions, and Property (4) ensures minimal information loss when interpreting FLM predictions. 


 
 

 

 

\begin{figure}[t]
        \centering
        \includegraphics[width = 0.4\textwidth]
        {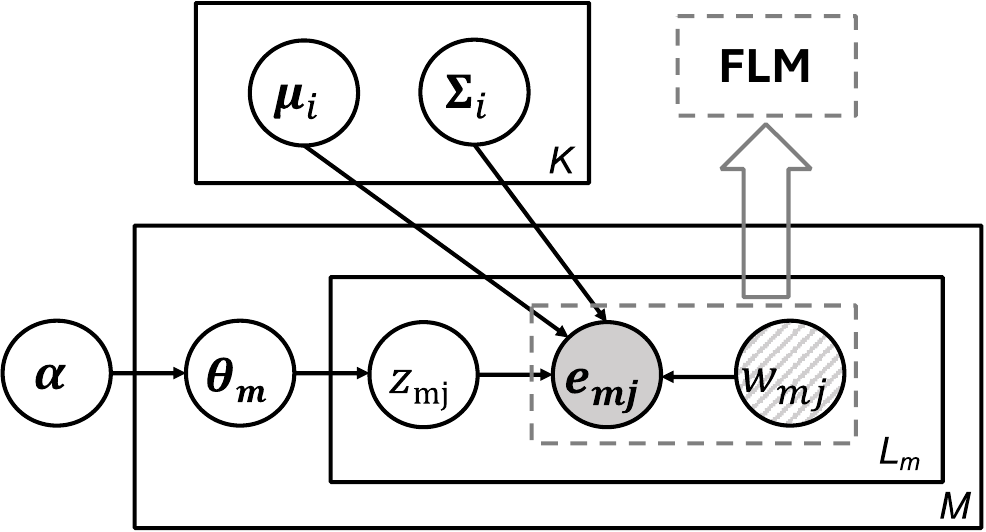}
        \caption{Graphical model of our VALC. The \emph{striped} circle represents \emph{continuous} word counts. }
        \label{fig:pgm}
\end{figure}
\subsection{VAriational Language Concepts (VALC)}\label{sec:valance}


\textbf{Method Overview.} 
Drawing inspiration from hierarchical Bayesian deep learning~\cite{BDL,BDLSurvey,NPN,CausalTrans,CounTS,VDI,PACE}, we propose our model, VAriational Language Concepts (VALC), to infer the optimal conceptual interpretation described in~\defref{def:concept}. 
Different from \emph{static} word embeddings~\citep{mikolov2013efficient} and topic models, FLMs produce \emph{contextual} word embeddings with continuous-value entries $[\e_{mj}]_{j=1}^{J_m}$ and more importantly, associate each word embedding with a continuous-value attention weight $[a_{mj}]_{j=1}^{J_m}$; therefore this brings unique challenges. 

To effectively discover latent concept structures learned by FLMs at the dataset level and interpret FLM predictions at the data-instance level, our VALC treats both the contextual word embeddings and their associated attention weights as observations to learn a probabilistic generative model of these observations, as shown in \figref{fig:overview}. The key idea is to use the attention weights from FLMs to compute a virtual continuous count for each word, and model the contextual word embedding distributions with Gaussian mixtures. 
The generative process of VALC is as follows (we mark key connection to FLMs in \blue{blue} and show the corresponding graphical model in~\figref{fig:pgm}): 

For each document $m,1\leq m \leq M$, 
\begin{enumerate}[nosep,leftmargin=18pt]
\item Draw the document-level concept distribution vector $\tha_m\sim \text{Dirichlet}(\alp)$.
\item For each word $j$ $(1\leq j \leq J_m)$,
\begin{enumerate}[nosep,leftmargin=10pt]
\item Draw the word-level concept index 
$z_{mj}\sim \text{Categorical}(\tha_m)$.
\item With a \blue{continuous} word count $w_{mj}\in \mathbb{R}$ from the \blue{FLM's attention weights}, draw the \blue{contextual word embedding} of the \blue{FLM} from the corresponding Gaussian component $\e_{mj}\sim \mathcal{N}(\muu_{z_{mj}},\Si_{z_{mj}})$.
\end{enumerate}
\end{enumerate}

Given the generative process above, discovery of latent concept structures in FLMs at the dataset level boils down to learning the parameters $\{\muu_k,\Si_k\}_{k=1}^K$ for the $K$ concepts. Intuitively the global parameters $\{\muu_k,\Si_k\}_{k=1}^K$ are shared across different documents, and they define a mixture of $K$ Gaussian distributions. Each Gaussian distribution describes a `cluster' of words and their contextual word embeddings.

Similarly, interpretations of FLM predictions at the data-instance level is equivalent to inferring the latent variables, i.e., document-level concept distribution vectors $\tha_m$ and word-level concept indices $z_{mj}$. Below we highlight several important aspects of our VALC designs.

\textbf{Attention Weights as \emph{Continuous} Word Counts.} 
Different from typical topic models~\citep{blei2003latent,blei2012probabilistic} and word embeddings~\citep{mikolov2013efficient} that can only handle \emph{discrete} word counts, our VALC can handle \emph{continuous} (virtual) word counts; this better aligns with continuous attention weights in FLMs. Specifically, we denote as $w_{mj}\in\RB_{\geq 0}$ the (non-negative real-valued) \emph{continuous word count} for the $j$'th word in document $m$. We explore three schemes of computing $w_{mj}$:
\begin{itemize}[nosep,leftmargin=18pt]
    \item \textbf{Identical Weights:} Use identical weights for different words, i.e., $w_{mj}=1, \forall m,j$. This is equivalent to typical discrete word counts. 
    \item \textbf{Attention-Based Weights with Fixed Length:} Use $w_{mj}= J'a_{mj}$, where $J'$ is a fixed sequence length shared across all documents. 
    \item \textbf{Attention-Based Weights with Variable Length:} Use $w_{mj}= \nicefrac{J_m a_{mj}}{\sum_{i=1}^{J_m}a_{mi}}$, where $J_m$ is true sequence length without padding. Note that in practice, $\sum_{i=1}^{J_m}a_{mi}\neq 1$ due to padding tokens in FLMs. 
\end{itemize}





\textbf{Contextual Continuous Word Representations.} 
Note that different from topic models~\citep{blei2003latent} and typical word embeddings~\citep{mikolov2013efficient,dieng2020topic} where word representations are \emph{static}, word representations in FLMs are \emph{contextual}; specifically, the same word can have different embeddings in different documents (contexts). For example, the word `soft' can appear as the $j_1$'th word in document $m_1$ and as the $j_2$'th word in document $m_2$, and therefore have two different embeddings (i.e., $\e_{m_1j_1}\neq\e_{m_2j_2}$). 

Correspondingly, in our VALC, 
we do not constrain the same word to have a static embedding; instead we assume that a word embedding is drawn from a Gaussian distribution corresponding to its latent topic. {Note} that word representations in our VALC is continuous, which is different from typical topic models~\citep{blei2003latent} based on (discrete) bag-of-words representations.

\subsection{Objective Function}\label{sec:learning}
Below we discuss the inference and learning procedure for VALC. We start by introducing the \emph{inference} of document-level and word-level concepts (i.e., $z_{mj}$ and $\tha_m$) given the global concept parameters (i.e., $\{(\muu_k,\Si_k)\}_{k=1}^K$), and then introduce the \emph{learning} of these global concept parameters.

\subsubsection{Inference}\label{sec:inf}

\textbf{Inferring Document-Level and Word-Level Concepts.} We formulate the problem of interpreting FLM predictions at the concept level as inferring document-level and word-level concepts. Specifically, given global concept parameters $\{(\muu_k,\Si_k)\}_{k=1}^K$, the \emph{contextual} word embeddings $\e_m \triangleq [\e_{mj}]_{j=1}^{J_m}$, and the associated attention weights $\a_m \triangleq [a_{mj}]_{j=1}^{J_m}$, a FLM produces for each document $m$, our VALC infers the posterior distribution of the document-level concept vector $\tha_m$, i.e., $p(\tha_m | \e_m,\a_m,\{(\muu_k,\Si_k)\}_{k=1}^K)$, and the posterior distribution of the word-level concept index $z_{mj}$, i.e., $p(z_{mj} | \e_m,\a_m,\{(\muu_k,\Si_k)\}_{k=1}^K)$. 

\textbf{Variational Distributions.} 
These posterior distributions are intractable; we therefore resort to variational inference~\citep{VI,blei2003latent} and use variational distributions $q(\tha_m | \gamm_m)$ and $q(z_{mj} | \ph_{mj})$ to approximate them. Here $\gamm_m\in \RB^K$ and $\ph_{mj} \triangleq [\phi_{mjk}]_{k=1}^K \in \RB^K$ are variational parameters to be estimated during inference. This leads to the following joint variational distribution: 
\begin{align}
&q(\tha_m, \{\z_{mj}\}_{j=1}^{J_m}|\gamm_m, \{\ph_{mj}\}_{j=1}^{J_m}) \nonumber\\
=~& q(\tha_m | \gamm_m) \cdot \prod\nolimits_{j=1}^{J_m}q(z_{mj} | \ph_{mj}). \label{eq:joint_vi}
\end{align}
\textbf{Evidence Lower Bound.} 
For each document $m$, finding the optimal variational distributions is then equivalent to maximizing the following evidence lower bound (ELBO): 
\begingroup\makeatletter\def\f@size{9.5}\check@mathfonts
\def\maketag@@@#1{\hbox{\m@th\large\normalfont#1}}
\begin{align}
     &\mathcal{L}(\gamm_m, \{\ph_{mj}\}_{j=1}^{J_m}; \alp, \{(\muu_k,\Si_k)\}_{k=1}^K) \nonumber\\
     =~&\EB_q[\log p(\tha_m|\alp)]+ \sum\nolimits_{j=1}^{J_m}\EB_q[\log p(z_{mj}|\tha_m)] \nonumber \\
    & + \sum\nolimits_{j=1}^{J_m}\EB_q[\log p(\e_{mj}|z_{mj},\muu_{z_{mj}}, \Si_{z_{mj}})] \nonumber\\
     &- \EB_q[\log q(\tha_m)] - \sum\nolimits_{j=1}^{J_m}\EB_q[\log q(z_{mj})], \label{eq:elbo}
\end{align}
\endgroup
where the expectation is taken over the joint variational distribution in~\eqnref{eq:joint_vi}. 

\textbf{Likelihood with \emph{Continuous} Word Counts.} 
One key difference between VALC and typical topic models~\citep{blei2003latent,blei2012probabilistic} is the virtual continuous (real-valued) word counts (discussed in~\secref{sec:valance}). Specifically, we define the likelihood in the third term of \eqnref{eq:elbo} as: 
\begingroup\makeatletter\def\f@size{7.5}\check@mathfonts
\def\maketag@@@#1{\hbox{\m@th\large\normalfont#1}}
\begin{align}
p(\e_{mj}|z_{mj},\muu_{z_{mj}}, \Si_{z_{mj}})=[\NM(\e_{mj};\muu_{mj},\Si_{mj})]^{w_{mj}}.\label{eq:con_count}
\end{align}
\endgroup
Note that \eqnref{eq:con_count} is the likelihood of $w_{mj}$ (virtual) words, where $w_{mj}$ is a {real} value derived from the FLM's attention weights (details in \secref{sec:valance}). 
{Therefore}, in the third item of \eqnref{eq:elbo}, we have: 
\begingroup\makeatletter\def\f@size{7.5}\check@mathfonts
\def\maketag@@@#1{\hbox{\m@th\large\normalfont#1}}
\begin{align}
    &\EB_q[\log p(\e_{mj}|z_{mj},\muu_{z_{mj}}, \Si_{z_{mj}})]  \nonumber\\
    =~&\sum\nolimits_{{k}} \phi_{mjk}w_{mj}\log\NM(\e_{mj}|\muu_k,\Si_k)\nonumber\\
    =~&\sum\nolimits_{{k}}\phi_{mjk} w_{mj} \{-\frac{1}{2}(\e_{mj}-\muu_k)^T\Si_k^{-1}(\e_{mj}-\muu_k) \nonumber\\
    &-\log[(2\pi)^{d/2} \vert \Si_k\vert^{1/2}]\}. \label{loss}
\end{align}
\endgroup
\textbf{Update Rules.} 
Taking the derivative of the ELBO in \eqnref{eq:elbo} w.r.t.  $\phi_{mjk}$ (see~\appref{sec:app_update_rules} for details) and setting it to $0$ yields the update rule for $\phi_{mjk}$:
\begingroup\makeatletter\def\f@size{7.5}\check@mathfonts
\def\maketag@@@#1{\hbox{\m@th\large\normalfont#1}}
\begin{align}  \label{eq:update_phi}
     \phi_{mjk} \propto~& \frac{w_{mj}}{\vert \Si_k\vert^{1/2}} 
      \exp[\Psi(\gamma_{mk})-\Psi(\sum\nolimits_{k'} \gamma_{mk'})  \nonumber\\
    & -\frac{1}{2}(\e_{mj}-\muu_k)^T\Si_k^{-1}(\e_{mj}-\muu_k)],
\end{align}  
\endgroup
with the normalization constraint $\sum\nolimits_{k=1}^K \phi_{mjk}=1$.
\begin{align}
    \gamma_{mk} = \alpha_k + \sum\nolimits_{j=1}^{J_m}  \phi_{mjk}w_{mj} ,\label{eq:update_gamma}
\end{align}
where $\alp\triangleq[\alpha_k]_{k=1}^K$ is the hyperparameter for the Dirichlet prior distribution of $\tha_m$. In summary, the inference algorithm will alternate between updating $\phi_{mjk}$ for all $(m,j,k)$ tuples and updating $\gamm_{mk}$ for all $(m,k)$ tuples. 


\subsubsection{Learning}\label{sec:learn}
\textbf{Learning Dataset-Level Concept Parameters.} The inference algorithm in~\secref{sec:inf} assumes availability of the dataset-level (global) concept parameters $\{(\muu_k,\Si_k)\}_{k=1}^K$. To learn these parameters, one needs to iterate between (1) inferring document-level variational parameters $\gamm_m$ as well as word-level variational parameters $\ph_{mj}$ in \secref{sec:inf} and (2) learning dataset-level concept parameters $\{(\muu_k,\Si_k)\}_{k=1}^K$. 


\textbf{Update Rules.} Similar to~\secref{sec:inf}, we expand the ELBO in~\eqnref{eq:elbo} (see~\appref{sec:app_update_rules} for details) and set its derivative w.r.t. $\muu_k$ and $\Si_k$ to $\0$,
yielding the update rule for learning $\muu_k$ and $\Si_k$:
\begingroup\makeatletter\def\f@size{7.5}\check@mathfonts
\def\maketag@@@#1{\hbox{\m@th\large\normalfont#1}}
\begin{align}
    \muu_k =~&  \frac{\sum_{m,j}{\phi_{mjk}w_{mj} \e_{mj}}}{\sum_{m,j} \phi_{mjk}w_{mj}}, \nonumber\\
    \Si_k =~& \frac{\sum_{m,j}\phi_{mjk}w_{mj} (\e_{mj}-\muu_k)(\e_{mj}-\muu_k)^T}{\sum_{m,j} \phi_{mjk}w_{mj} }. \label{eq:update_mu_sigma}
\end{align}
\endgroup

\begin{algorithm}[h]
 \caption{Algorithm for VALC}\label{alg:valance}
 \textbf{Input:} Initialized $\{\gamm_m\}_{m=1}^M$, $\{\ph_m\}_{m=1}^M$, and $\{\Om_k\}_{k=1}^K$, documents $\{\DM_m\}_{m=1}^M$, number of epochs~T.
 
 \textbf{for} {$t=1:T$} \textbf{do}{
 
 \quad\textbf{for} {$m=1:M$} \textbf{do}{
 
 \quad\quad Update $\ph_m$ and $\gamm_m$ using \eqnref{eq:update_phi} and \eqnref{eq:update_gamma}, respectively.
 
 }
 \quad Update $\{\Om_k\}_{k=1}^K$ using \eqnref{eq:update_mu_sigma}.
 }
\end{algorithm}
\textbf{Effect of Attention Weights.} 
From \eqnref{eq:update_mu_sigma}, we can observe that the attention weight of the $j$'th word in document $m$, i.e., $a_{mj}$, affects the virtual continuous word count $w_{mj}$ (see~\secref{sec:valance}), thereby affecting the update of the dataset-level concept center $\muu_k$ and covariance $\Si_k$. 
Specifically, if we use attention-based weights with fixed length or variable length in~\secref{sec:valance}, the continuous word count $w_{mj}$ will be proportional to the attention weight $a_{mj}$. Therefore, when updating the concept center $\muu_k$ as a weighted average of different word embeddings $\e_{mj}$, VALC naturally places more focus on words with higher attention weights $a_{mj}$ from FLMs, thereby making the interpretations sharper (see~\secref{sec:interpret} for detailed results and~\appref{app:theory} for theoretical analysis). 


\section{Theoretical Analysis} \label{sec:concept_theory}

{In this section, we provide theoretical guarantees {of VALC on the four properties} in~\defref{def:concept}. 

\textbf{Multi-Level Structure.} 
As shown in~\algref{alg:valance}, VALC (1) learns the \emph{dataset-level} interpretation $\{\Om_k\}_{k=1}^K$ describing the $K$ concepts, (2) infers the distribution of \emph{document-level} interpretation $\tha_m$ for document $m$, i.e., $q(\tha_m|\gamm_m)$ {(parameterized by $\gamm_m$)}, and (3) infers the posterior distribution of \emph{word-level} concept index, i.e., $q(z_{mj}|\ph_{mj})$, parameterized by $\ph_{mj}$. Such three-level interpretations correspond to Property (1) in~\defref{def:concept}. 


\textbf{Normalization.} 
The {learned} variational distribution $q(\tha_m|\gamm_m)$ (described in \eqnref{eq:joint_vi}) is a Dirichlet distribution; therefore we have $\sum\nolimits_{k=1}^{K} \theta_{mk}  = 1$. The update of $\ph_{mj}$ (\eqnref{eq:update_phi}) is naturally constrained by $\sum\nolimits_{k=1}^K \phi_{mjk}=1$ since $\ph_{mj}$ parameterizes a Categorical distribution (over $z_{mj}$). 

   
\begin{algorithm}[t]
 \caption{Algorithm for VALC Concept Editing }\label{alg:prune_valance}
 \textbf{Input:} FLM $f(\cdot)$, classifier $g(\cdot)$, classification loss $L$, document $\DM_m$ with $J_m$ words, labels $\y$, constant factor $\omega$.\\
 \textbf{for} {$j=1:J_m$} \textbf{do}{
\\\quad $\e_{mj} = f(\DM_{mj})$ 

\quad $\x^* = QP(\e_{mj}, \{\muu_k\}_{k=1}^{K})$ 

\quad $k^* = \arg \min L(g(\e_{mj} - \omega\cdot x^*_k \muu_{k}), y_m)$ 

\quad $\e_{mj} \leftarrow  \e_{mj} - \omega\cdot x^*_{k^*} \muu_{k^*}$
 }
\end{algorithm}
\textbf{Additivity}.} 
VALC is able to perform \emph{\textbf{Concept Editing}}, i.e, add/subtract the learned concept  
activation $\mu_k$ from FLMs via the following Quadratic Programming (QP) problem ($\x=[x_k]_{k=1}^K$):
\begin{align} \label{eq:qp}
	\min\nolimits_{\x \in \mathbb{R}^K} \quad & \Vert  \sum\nolimits_{k=1}^{K}x_k\muu_k -\e_m\Vert ^2,  \nonumber\\
	 \qquad\text{subject to}\quad & \x \geq \0   
	 \mbox{~~and~~}
           \sum\nolimits_{k=1}^K x_k = 1.
\end{align}
Given learned concepts $\{(\muu_k,\Si_k)\}_{k=1}^K$, VALC obtains this QP's optimal solution $\x^* \in \mathbb{R}^K$ and add/subtract any concept $k$ from arbitrary FLM embedding $\e_m$ by:
$
    \e_m \leftarrow \e_m \pm  x^*_{k} \muu_k.
$ 
\algref{alg:prune_valance} summarizes this {\emph{concept editing} process; one can also replace $\e_{mj}$ with the CLS embedding $\c_{m}$ for document-level editing (details in~\appref{app:document_edit}). 

\textbf{Mutual Information Maximization.} 
\thmref{thm:MI_bound} below shows that our inferred document-level and  word-level interpretation, $\tha_m$ and $\{\ph_{mj}\}_{j=1}^{J_m}$, satisfy Property (4), Mutual Information Maximization, in~\defref{def:concept}. 

\begin{theorem}[\textbf{Mutual Information Maximization}]
\label{thm:MI_bound}
In~\eqnref{eq:elbo}, the ELBO $\mathcal{L}(\gamm_m, \{\ph_{mj}\}_{j=1}^{J_m}; \alp, \{(\muu_k,\Si_k)\}_{k=1}^K)$ is upper bounded by the mutual information between contextual embeddings $\e_m$ and multi-level interpretation $\tha_m,\{\ph_{mj}\}_{j=1}^{J_m}$ in~\defref{def:concept}. Formally, with approximate posteriors $q(\tha_m | \gamm_m)$ and $q(z_{mj} | \ph_{mj})$, we have
\begin{align}\label{eq:MI_bound}
\mathcal{L}&(\gamm_m, \{\ph_{mj}\}_{j=1}^{J_m}; \alp, \{(\muu_k,\Si_k)\}_{k=1}^K) \nonumber\\
     \leq~& {I}(\e_m; \tha_m,\{z_{mj}\}_{j=1}^{J_m})-H(\e_m),
\end{align}
where the entropy term $H(\e_m)$ is a constant. 
\end{theorem}

From \thmref{thm:MI_bound} we can see that maximizing the ELBO in~\eqnref{eq:elbo} is equivalent to maximizing the mutual information between our document-level/word-level concepts and the observed contextual embeddings in FLMs ({{proof in}~\appref{app:proof}}).

In summary, VALC enjoys all four properties in~\defref{def:concept} and therefore generates the optimal conceptual interpretation for FLMs. In contrast, state-of-the-art methods only satisfy a small part of them (\tabref{table:property_compare} and~\secref{sec:exp_property}). 
In~\appref{app:theory}, we provide theoretical guarantees that (1) under mild assumptions our VALC can learn better conceptual interpretations for FLMs for in noisy data and (2) 
attention-based schemes is superior to the identical scheme ({described in} \secref{sec:valance}). 

\section{Experiments}

        

\subsection{Experiment Setup}\label{sec:setup}

\textbf{Datasets.} 
We use three datasets in
our experiments, namely 20 Newsgroups, M10~\citep{lim2015m10}, and BBC News~\citep{greene2006bbcnews}. {For preprocessing details, see~\appref{app:experiment}.}



\textbf{Baselines.} 
We compare our method with 
the following state-of-the-art baselines: 
\begin{itemize}[nosep,leftmargin=18pt]
\item 
{\textbf{SHAP and LIME}~\citep{lundberg2017shap,ribeiro2016lime} are interpretation methods that attribute importance scores to input features. In this paper, we use embeddings of `CLS' token as input to SHAP/LIME.}
\item \textbf{BERTopic}~\citep{grootendorst2020bertopic} is a clustering-based model that uses HDBSCAN~\citep{mcinnes2017hdbscan} to cluster sentence embeddings from BERT, performs Uniform Manifold Approximation Projection (UMAP)~\citep{mcinnes2018umap}, and then uses class-based TF-IDF (c-TF-IDF) to obtain words for each cluster. 
\item \textbf{CETopic}~\citep{zhang2022cetoic} is a clustering-based model that first uses UMAP to perform dimensionality reduction on BERT sentence embeddings, performs K-Means clustering~\citep{lloyd1982kmeans}, and then uses weighted word selection for each cluster.
\end{itemize}





\textbf{Evaluation Metric.} 
Inspired by~\citet{koh2020CBM}, we perform concept editing experiments to evaluate conceptual interpretation for FLMs; 
higher \emph{accuracy gain} after editing indicates better interpretation performance. 
{We leverage BERT-base-uncased~\citep{devlin2018bert} as the contextual embedding model, 
and use accuracy on the test set as our metric. {For a fair comparison, we adhered to the baseline methodologies (e.g., BERTopic and CETopic) by setting the number of concepts (topics)
$K$ to 100 across all datasets. This number was chosen as it strikes an effective balance between capturing adequate detail and avoiding model overfitting.} See~\appref{app:concept_edit} for more details.} 





We can perform concept editing on either input tokens or contextual embeddings of FLMs. 
Specifically, we can perform \emph{hard} concept editing for concept $k$ by directly removing tokens that belong concept $k$ (applicable for hard clustering methods such as our baselines); we could also perform \emph{soft} concept editing for concept $k$ by removing concept subspace vectors from contextual embeddings $\e_m$ (applicable for VALC using~\algref{alg:prune_valance}). 

{Following~\cite{lyu2024towards}, we conducted additional experiments to evaluate the faithfulness metric. The faithfulness metric is implemented as the accuracy score of predictions using logistic regression, with the inferred conceptual explanations as inputs. }
        

\subsection{Comparison on Four Properties in~\defref{def:concept}} \label{sec:exp_property}
\begin{table}[t]
      \caption{Comparing methods on the properties in~\defref{def:concept} (MIM: Mutual Information Maximization).}\label{table:property_compare}
        \footnotesize
      \centering
      \LARGE
      \resizebox{0.47\textwidth}{!}{
      \begin{tabular}{lcccc}
   
        \toprule
        Model  & Multi-Level & Normalization    & Additivity & MIM    \\
        \midrule
        SHAP/LIME & No & No &  Partial &  No\\
        BERTopic & No & Hard & Partial & No \\
        
        CETopic & No & Hard &  Partial &  No\\
        VALC  & \textbf{Yes} & \textbf{Soft} & \textbf{Full} & \textbf{Yes} \\
        \bottomrule
      \end{tabular}
      }
\end{table}

In~\secref{sec:concept_theory} we show that VALC satisfies the four properties of conceptual interpretation in~\defref{def:concept}. In contrast, baseline models do not necessarily learn concepts that meet these requirements. 
\tabref{table:property_compare} summarizes the comparison between VALC and the baselines. We can see that VALC is superior to baselines in terms of the following four aspects: 
\begin{enumerate}[nosep,leftmargin=18pt]
    \item[(1)]\textbf{Multi-Level Structure.} 
    Baselines either apply clustering algorithms directly on the document-level embeddings from FLMs or assign importance scores to input features, and thus can only provide single-level interpretation, necessitating complex post-processing to generate dataset-level concepts. In contrast, VALC adopts an integrated approach, learning concepts at the dataset, document, and word level in a joint, end-to-end manner. 
    \item[(2)] \textbf{Normalization.} {BERTopic and CETopic} assign each word to exactly one concept and therefore satisfies \emph{hard}-normalization. SHAP/LIME produce importance scores that are not normalized. In contrast, VALC learns fractional concept interpretations $\gamm_m$ and $\ph_{mj}$ and therefore satisfies \emph{soft}-normalization, which is more flexible and intuitive. 
    
    \item[(3)] \textbf{Additivity.} 
    Baselines perform 
    addition or subtraction of concepts only at a single level (word/document), while our additivity and concept editing (\algref{alg:prune_valance}) work for both levels. 
    \item[(4)] \textbf{Mutual Information Maximization.} 
    Baselines {either} use a multi-step pipeline or produce importance scores; they are therefore prone to lose information between {FLM embeddings} and final clustering{/scoring} results. In contrast, VALC is theoretically guaranteed to maximally preserve information (\thmref{thm:MI_bound}). 
\end{enumerate}
        

\begin{table}[t]

      \caption{\textbf{Accuracy gain on 20 Newsgroups (20NG), M10, and BBC News (BBC) (\%).} We mark the best results with \textbf{bold face} and {the second best} with \underline{underline}.}
      \label{table:prune_dataset}
        \footnotesize
      \centering
      
      \resizebox{0.48\textwidth}{!}{
      \Large
      \begin{tabular}{c|l|c|cccc|c}
        \toprule
        \multirow{2}{*}{Dataset}&&\multirow{2}{*}{Unedited} &{SHAP}&\multirow{2}{*}{BERTopic}&\multirow{2}{*}{CETopic}&\multirow{2}{*}{VALC }& Finetune \\
        && &/LIME&&&&(Oracle)   \\
        \midrule
        \multirow{2}{*}{20NG}&Acc. & 51.26 & 61.74&  60.76   & \underline{61.93} & \bf{62.54} & 64.38  \\
        &Gain & -  & 10.48& 9.50   & \underline{10.67} & \bf{11.28} & 13.12 \\
        \midrule
        \multirow{2}{*}{M10}&Acc. &  69.74  & 75.60 &   76.79  &  \underline{79.18} & \bf{80.74} & 82.54  \\
        &Gain & -  & 5.86 & 7.05 & \underline{9.44} & \bf{11.00} & 12.80 \\
        \midrule
        \multirow{2}{*}{BBC}&Acc. & 93.72 & 95.96 & 95.52  &  \bf{96.86} & \underline{96.41} & 97.76  \\
        &Gain & - & 2.24 & 1.80 & \bf{3.14} & \underline{2.69} & 4.04 \\
        \bottomrule
      \end{tabular}
      }
       \vskip -0.5cm
\end{table}
\subsection{{Results}} \label{sec:exp_prune}
\textbf{Accuracy Gain.} 
We perform greedy concept editing~\citep{koh2020CBM} for BERTopic, CETopic, and our VALC to evaluate the quality of their learned concepts. Higher accuracy gain after pruning indicates better performance. 

\tabref{table:prune_dataset} show the results for different methods in three real-world datasets, where `Finetune (Oracle)' refers to finetuning both the backbone and the classifier of BERT. VALC's concept editing can improve the accuracy upon the unedited model by more than $11\%$ in 20 Newsgroups and M10, almost on par with `Finetune (Oracle)'. Compared with the baselines, VALC achieves the most accuracy gain in 20 Newsgroups and M10 and the second most accuracy gain in BBC News, demonstrating the effectiveness of VALC's four properties in~\defref{def:concept}. Note that SHAP and LIME both interpret the CLS token's embedding and therefore has identical accuracy gain (details in~\appref{app:concept_edit}).


\begin{figure*}[!t]
        \centering
        \includegraphics[width = 1.0\textwidth]{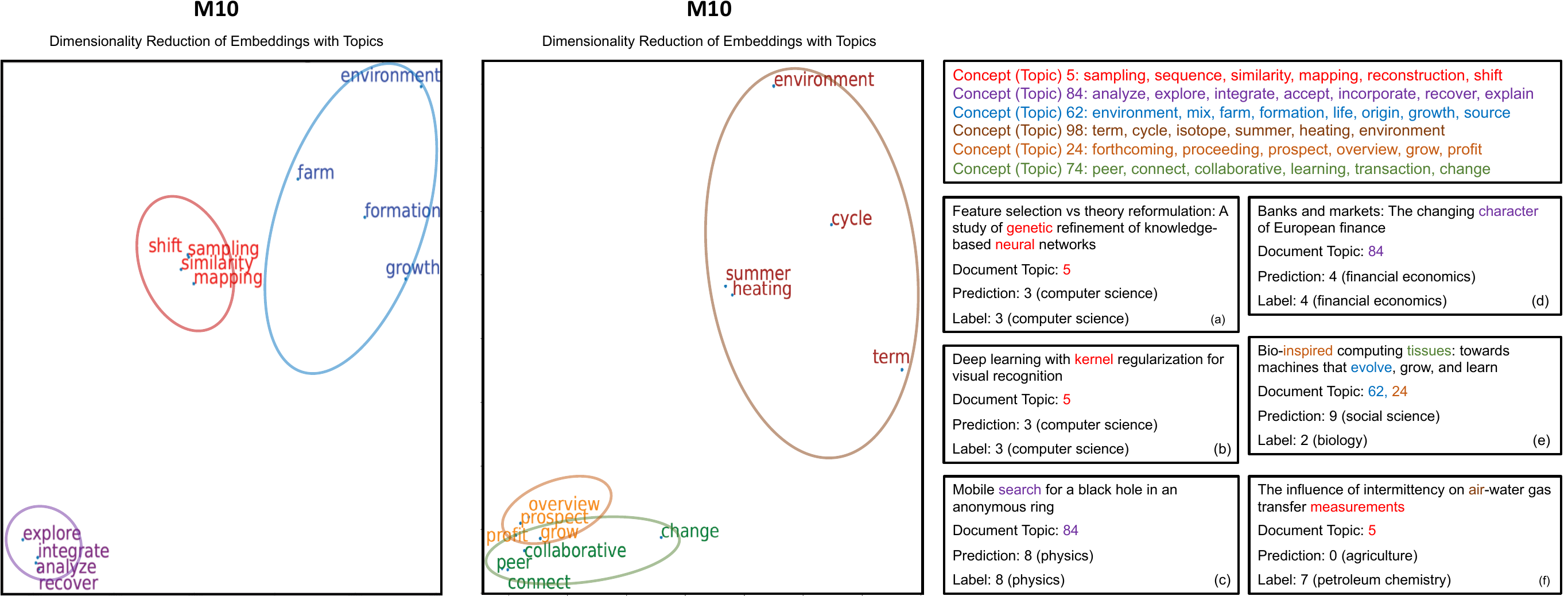}
        \vskip -0.3cm
        \caption[width = 0.8\linewidth]{Visualization of VALC's three-level conceptual interpretation. \textbf{Left and Middle:} Dataset-level interpretation with $6$ concepts' $\muu_k$ and $\Si_k$ with nearest word embeddings ($3$ concepts per plot for clarity). \textbf{Right:} Top words in each concept and $6$ example documents with the associated document-level and word-level interpretations. }
        \label{topic-m10}
\end{figure*}

\textbf{Ablation Study.} 
Thanks to its full additivity (\defref{def:concept}), VALC is capable of different concept editing schemes, including `Random', `Unweighted', and `Weighted'. Specifically, \emph{weighted} pruning uses the concept editing algorithm in~\algref{alg:prune_valance} with the optimal hyperparameter $\omega$; \emph{unweighted} pruning runs~\algref{alg:prune_valance} with $\omega=1$; \emph{random} pruning first randomly picks a {concept $k$ ($k\in \{1,...,K\}$)}, sets $\omega \cdot x_k=1/K$, and then runs~\algref{alg:prune_valance}.  
\begin{table}[t]
      \caption{\textbf{VALC Editing Accuracy (\%).} We mark the best results with \textbf{bold face}, {second best with \underline{underline}}. }\label{table:prune_ablation}
      \vskip -0.2cm
        \footnotesize
      \centering
      \resizebox{0.48\textwidth}{!}{
      \LARGE
      \begin{tabular}{l|c|ccc|c}
        \toprule
        \multirow{2}{*}{Dataset} &\multirow{2}{*}{Unedited} &\multirow{2}{*}{Random}&\multirow{2}{*}{Unweighted}&\multirow{2}{*}{Weighted}& Finetune \\
        & &&&&(Oracle)   \\
        \midrule
        20 Newsgroups & 51.26 &   51.13  & \underline{54.63} & \bf{62.54} & 64.38  \\
        M10 & 69.74 & 69.76  & \underline{73.56} & \bf{80.74} & 82.54 \\
        BBC News& 93.72 & 93.72 & \underline{95.52} & \bf{96.41} & 97.76 \\       
        \bottomrule
      \end{tabular}
      }
       \vskip -0.3cm
\end{table}
\tabref{table:prune_ablation} shows accuracy for VALC's different schemes. As expected, random pruning barely improves upon the unedited model. 
Unweighted pruning improves upon the unedited model by $1.5\sim 3.5\%$. 
Weighted pruning improves the accuracy by around $11\%$ upon the unedited model on 20 Newsgroups and M10. 

\textbf{{Faithfulness.}}
\tabref{tab:faith} shows the faithfulness of VALC and baselines on the 20 Newsgroups, M10, and BBC News datasets. These results show that our VALC significantly outperforms the baseline models, achieving the highest faithfulness accuracy scores in the 20 Newsgroups ($89.8\%$), M10 ($99.5\%$), and BBC News ($100.0\%$) datasets.

Note that the dataset size of 20 Newsgroups, M10, and BBC News is $16{,}309$, $8{,}355$, and $2{,}225$, respectively. BBC News contains significantly less data, making it easier to achieve a high faithfulness score. This explains why both CETopic and our VALC obtain a faithfulness score of $100.0\%$.

Baseline methods such as BERTopic and CETopic represent language concepts as discrete bags of words, which lack flexibility and accuracy. In contrast, VALC infers continuous concepts for datasets, documents, and words with theoretical guarantees. Consequently, it provides optimal and faithful conceptual explanations of high quality.

\begin{table}[t]
\centering
\vskip -0.05cm
\caption{Additional results for the faithfulness (in terms of accuracy percentage (\%)) of VALC and baselines on the 20 Newsgroups, M10, and BBC News datasets. We mark the best results with \textbf{bold face}.}
\resizebox{0.48\textwidth}{!}{
\begin{tabular}{lccccc}
\hline
{Methods} & {20 NG} & {M10} &{BBC} & Average 
\\ \hline
SHAP/LIME        & 5.8                    & 13.9         & 22.9  & 14.2       \\
BERTopic         & 17.2                    & 87.6         & 64.6 &   56.5     \\
CETopic          & 79.2                    & 96.4         & \bf{100.0} & 91.9       \\
\midrule
VALC             & \bf{89.8}          & \bf{99.5} & \bf{100.0} & \bf{96.4}
\\ \hline
\end{tabular}}
\label{tab:faith}
\vskip -0.5cm
\end{table}

{See~\appref{app:quantitative} for more quantitative results. 

\subsection{Conceptual Interpretation ({More for Different Tasks in~\appref{app:interpret_more}})}
\label{sec:interpret}
\textbf{Dataset-Level Interpretations.} 
As a case study, we train VALC on M10, sample $6$ concepts (topics) from the dataset, and plot the word embeddings of the top words (closest to the center $\muu_k$) in these concepts using PCA in \figref{topic-m10}(left and middle). 
We can observe \red{Concept 5} is mostly about \red{data analysis}, including words such as `sampling' and `similarity'. 
\purple{Concept 84} is mostly about \purple{reasoning}, with words `explore', 'accept', `explain', etc. \blue{Concept 62} is mostly about \blue{nature}, with words `environment', `formation', `growth', etc. \brown{Concept 98} is mostly about \brown{farming}, with words `term', `summer', `heating', etc. \orange{Concept 24} is mostly about \orange{economics}, with words `forthcoming', `prospect', `grow', etc. \green{Concept~74} is mostly about \green{social contact}, containing words such as `peer', `connect', and `collaborative'. 
Interestingly, \orange{Concept 24 (economics)} and \green{Concept~74 (social contact)} are both related to social science and are therefore closer to each other in \figref{topic-m10}(middle), while \brown{Concept 98 (farming)} is farther away, showing VALC's cability of capturing concept similarity. 

\textbf{Document-Level Interpretations.} 
\figref{topic-m10}(right) shows 
that VALC can provide conceptual interpretations on why correct or incorrect FLM predictions happen for specific documents. 
For example, document (e) belongs to class 2 (\textbf{biology}), but BERT misclassifies it as class 9 (\textbf{social science}); our VALC interprets that this is because document (e) involves \orange{Concept 24 (economics)}, which is related to \textbf{social science}. 
On the other hand, document (b) is related to machine learning and BERT correctly classifies it as class 3 (\textbf{computer science}); VALC interprets that this is because document (b) involves \red{Concept 5 (data analysis)}. 


\textbf{Word-Level Interpretations.} \figref{topic-m10}(right) also shows that VALC can interpret which words and what concepts of these words lead to specific FLM predictions. 
For example, document~(f) belongs to class 7 (\textbf{petroleum chemistry}), but BERT misclassifies it as class 0 (\textbf{agriculture}); VALC attributes this to the word `air', 
which belongs to \brown{Concept 98 (farming)}. 
For document~(b), VALC interprets that BERT correctly classifies it as class 3 (\textbf{computer science}) because the document contains the word `kernel' that belongs to \red{Concept 5 (data analysis)}.


 

\section{Conclusion}
{We address the challenge of multi-level interpretations for FLM predictions by defining conceptual interpretation and introducing VALC, the first method to infer such interpretations effectively. Empirical results are promising, and theoretical analysis confirms that VALC reliably produces optimal conceptual interpretations by our definition. }

\section{Limitations}
{Our proposed method assumes access to the hidden layers of Transformer-based models, and therefore can be naturally extended to Transformer-based models including RoBERTa~\cite{liu2019roberta}, DeBERTa~\cite{he2021deberta}, ALBERT~\cite{lan2019albert}, Electra~\cite{clark2020electra}, {and decoder-only models, such as GPTs~\cite{radford2019language,brown2020language}}. Although our VALC is initially designed for Transformer-based models, it is also generalizable to other architectures, such as Convolutional Neural Networks (CNNs)~\cite{lecun2015deep} and Long Short-Term Memory networks (LSTMs)~\cite{hochreiter1997long}, by simply setting identical attention weights. Future work may include extending VALC beyond Transformer variants and natural language applications. {However, many other foundation language models provided by proprietary sources may not expose their internal states, limiting the applicability of our method in such cases.}} 
\section{{Ethical Considerations}}
VALC, as the first to comprehensively interpret FLMs at the concept level, holds significant promise for advancing societal and technological progress. By elucidating the inner workings of these complex FLMs, we enable greater transparency and trust in AI systems, which is crucial for their widespread adoption. This transparency ensures that AI-driven decisions in critical areas such as healthcare, law, and finance are more explainable and accountable, thus safeguarding against biases and errors. Additionally, our VALC fosters enhanced collaboration between AI and human experts, as interpretable models can provide insights that are more easily understood and acted upon by domain specialists. This symbiotic relationship has the potential to accelerate innovation, improve decision-making processes, and ultimately lead to more ethical and equitable AI applications, thereby benefiting society at large.

\section{{Acknowledgements}}
We extend our sincere gratitude to Akshay Nambi and Tanuja Ganu from Microsoft Research for their invaluable insights and guidance, which significantly enhanced the quality of this work. We also express our deep appreciation to Microsoft Research AI \& Society Fellowship, NSF Grant IIS-2127918, NSF CAREER Award IIS-2340125, NIH Grant 1R01CA297832, and the Amazon Faculty Research Award for their generous support. This research is also supported by NSF National Artificial Intelligence Research Resource (NAIRR) Pilot and the Frontera supercomputer supported by the National Science Foundation (award NSF-OAC 1818253) at the Texas Advanced Computing Center (TACC) at The University of Texas at Austin. Additionally, we thank the anonymous reviewers and the area chair/senior area chair for their thoughtful feedback and for recognizing the significance and contributions of our research. Finally, we would like to thank the Center for AI Safety (CAIS) for providing the essential computing resources that enabled this work.

\bibliography{main}

\appendix
\onecolumn

\section{Details on Learning VALC} \label{sec:app_update_rules}
\textbf{Update Rules.} Similar to Sec. 3.4.1 of the main paper, we expand the ELBO in~\eqnref{eq:elbo} of the main paper, take its derivative w.r.t. $\muu_k$ and set it to $\0$:  
\begin{align}
    \frac{\partial L}{\partial \muu_k} =\sum_{m,j}\phi_{mjk}w_{mj} \Si_{k}^{-1} (\e_{mj}-\muu_k) = 0,
\end{align}
yielding the update rule for learning $\muu_i$:
\begin{align}
    \muu_k &=  \frac{\sum_{m,j}{\phi_{mjk}w_{mj} \e_{mj}}}{\sum_{m,j} \phi_{mjk}w_{mj}},\label{eq:update_mu_app}
\end{align}
where $\Si_{k}^{-1}$ is canceled out. 
Similarly, setting the derivatives w.r.t. $\Si$ to $\0$, i.e., 
 \begin{align}
    \frac{\partial L}{\partial \Si_k} 
    = \frac{1}{2}\sum_{m,j} \phi_{mjk}w_{mj}(-\Si_k^{-1}
    +\Si_k^{-1}(\e_{mj}-\muu_k)(\e_{mj}-\muu_k)^T\Si_k^{-1}),
 \end{align}
we have
\begin{align}
     \Si_k &= \frac{\sum_{m,j}\phi_{mjk}w_{mj} (\e_{mj}-\muu_k)(\e_{mj}-\muu_k)^T}{\sum_{m,j} \phi_{mjk}w_{mj} }. \label{eq:update_sigma_app}
\end{align}
\begin{figure}[h]
        \vskip -0.0cm
        \centering
        \includegraphics[width = 0.4\textwidth]
        {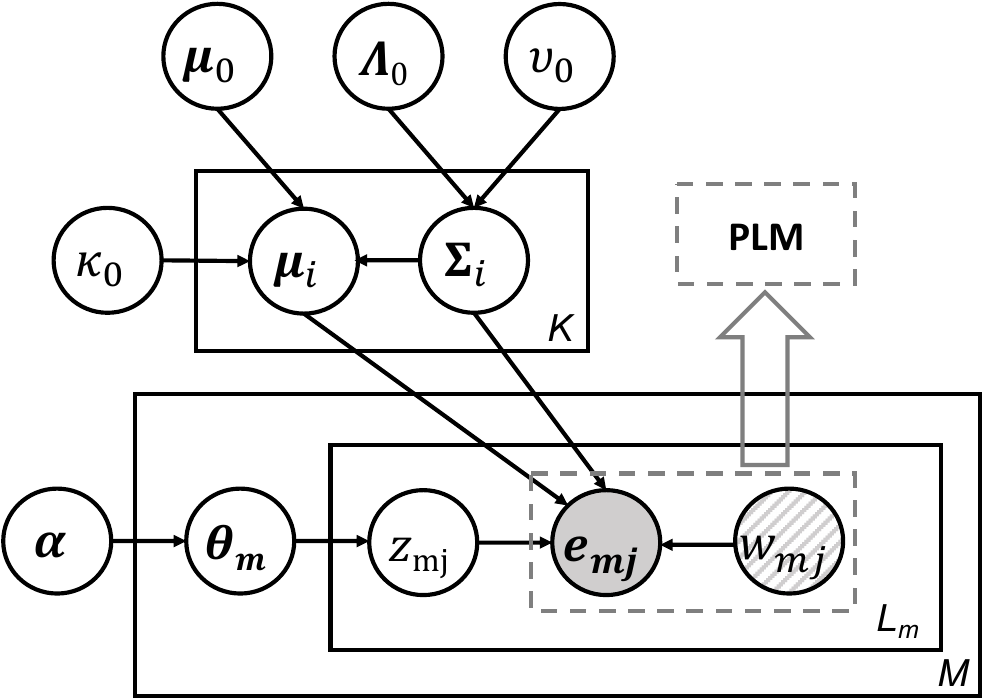}
        \vskip -0.3cm
        \caption{Probabilistic graphical model of smoothed VALC. }
        \label{fig:pgm_smoothed}
       \vskip -0.5cm
\end{figure}

\textbf{Smoothing with Prior Distributions on $\{(\muu_k,\Si_k)\}_{k=1}^K$.} 
To alleviate overfitting and prevent singularity in numerical computation, we impose priors distributions on $\muu_k$ and $\Si_k$ to smooth the learning process (\figref{fig:pgm_smoothed}). Specifically, we use a Normal-Inverse-Wishart prior on $\muu_i$ and $\Si_i$: 
\begin{align*}
    \Si_k &\sim \mathcal{IW}(\Lam_0,\nu_0),\\
    \muu_k|\Si_k &\sim \mathcal{N}(\muu_0,\Si_k/\kappa_0),
    \end{align*}
where $\Lam_0$, $\nu_0$, $\muu_0$, and $\kappa_0$ are hyperparameters for the prior distributions. 
Taking the expectations of $\muu_k$ and $\Si_k$ over the posterior distibution $\mathcal{NIW}(\muu_k,\Si_k|\muu_k^{(n)},\Lam_k^{(n)},\kappa_k^{(n)},\nu_k^{(n)})$, we have the update rules as:
\begin{align}
    \muu_k &\leftarrow \EB_{\mathcal{NIW}}[\muu_k] 
    = \frac{\kappa_0\muu_0+n_k\tilde{\muu}_k}{\kappa_0+n_k},\label{eq:update_smooth_mu}\\
    \Si_k &\leftarrow \EB_{\mathcal{NIW}}[\Si_k] 
    = \frac{\Lam_0+\S_k + \frac{\kappa_0 n_k}{\kappa_0+n_k}(\tilde{\muu}_k-\muu_0)(\tilde{\muu}_k-\muu_0)^T}{\nu_0+n_k-K-1},\label{eq:update_smooth_sigma}\\
    \S_k &= \sum\nolimits_{m,j}\phi_{mjk}w_{mj} (\e_{mj}-\tilde{\muu}_k)(\e_{mj}-\tilde{\muu}_k)^T.
\end{align}
where $n_k=\sum_{m,j}\phi_{mjk}w_{mj}$ is the total virtual word counts used to estimate $\muu_k$ and $\Si_k$. \eqnref{eq:update_smooth_mu} and~\eqnref{eq:update_smooth_sigma} are the smoothed version of~\eqnref{eq:update_mu_sigma} of the main paper. From the Bayesian perfective, they correspond to the expectations of $\muu_k$'s and $\Si_k$'s posterior distributions. \algref{alg:valance} of the main paper summarizes the learning of VALC. 

\textbf{Online Learning of $\muu_k$ and $\Si_k$.} 
Note that FLMs are deep neural networks trained using minibatches of data, while~\eqnref{eq:update_smooth_mu} and~\eqnref{eq:update_smooth_sigma} need to go through the whole dataset before each update. Inspired by~\citet{hoffman2010online,VQVAE}, we using exponential moving average (EMA) to work with minibatchs. Specifically, we update them as: 
\begin{align*}
\muu_k &\leftarrow \rho \cdot N \cdot\muu_k + (1-\rho)\cdot B\cdot\tilde{\muu}_{k},\\
\Si_k &\leftarrow \rho \cdot N \cdot\Si_k + (1-\rho)\cdot B\cdot\tilde{\Si}_{k},\\
N &\leftarrow \rho \cdot N  + (1-\rho)\cdot B,\\
\muu_k &\leftarrow \frac{\muu_k}{N}, \quad \Si_k \leftarrow \frac{\Si_k}{N},
\end{align*}
where $B$ is the minibatch size, $N$ is a running count, and $\rho\in(0,1)$ is the momentum hyperparameter. $\tilde{\muu}_{k}$ and $\tilde{\Si}_{k}$ are the updated $\muu_k$ and $\Si_k$ after applying~\eqnref{eq:update_smooth_mu} and~\eqnref{eq:update_smooth_sigma} only on the \emph{current minibatch}.

\textbf{Effect of Attention Weights.}
{Interestingly, we also observe that FLMs' attention weights on stop words such as `the' and `a' tend to be much lower; therefore VALC can naturally ignore these concept-irrelevant stop words when learning and inferring concepts (as discussed in~\secref{sec:learn}). This is in contrast to typical topic models~\citep{blei2003latent,blei2012probabilistic} that require preprocessing to remove stop words.}

\textbf{Phrase-Level Interpretations.} 
{We can easily infer phrase-level concepts from word-level concepts by treating phrases as sub-documents and adapting~\eqnref{eq:update_gamma} (which provides document-level concepts) in the paper. Specifically, suppose for a given phrase spanning from the $r$-th word to the $s$-th word in document $m$, we can adapt~\eqnref{eq:update_gamma} to provide phrase-level conceptual explanations as $\gamma_{mk}^{(r,s)} = \alpha_k + \sum_{j=r}^s \phi_{mjk} w_{mj}$. Here $\gamma_{mk}^{(r,s)}$ is the strength of concept $k$ for the given phrase in document $m$. 
In this way, $\gamma_{mk}^{(r,s)}$ can serve as the phrase-level concept explanation of the phrase spanning from $r$-th word to the $s$-th word; this is another interesting complementary sub-document-level concept explanation between the word level and the document level. }
\section{{Interpretation of the ELBO}}\label{sec:app_elbo_expansion}
{VALC's evidence lower bound (ELBO), i.e.,~\eqnref{eq:elbo} in the paper, is}  
\begin{align}\label{eq:elbo_copy}
    \mathcal{L}(\gamm_m, \{\ph_{m[1:J_m]}\}; \alpha, \{(\muu_{[1:K]},\Si_{[1:K]})\}) 
    & =\mathbb{E}_q[\log p(\tha_m|\alpha)]+ \sum\nolimits_{j=1}^{J_m}\mathbb{E}_q[\log p(\z_{mj}|\tha_m)]\nonumber\\
   &+ \sum\nolimits_{j=1}^{J_m}\mathbb{E}_q[\log p(\e_{mj}|\z_{mj},\mu_{\z_{mj}}, \Sigma_{\z_{mj}})] \nonumber\\
     &- \mathbb{E}_q[\log q(\tha_m)] - \sum\nolimits_{j=1}^{J_m}\mathbb{E}_q[\log q(\z_{mj})].
\end{align}

{\textbf{Derivation of the Evidence Lower Bound.} We derive the evidence lower bound by computing the log likelihood of each term. For example, by definition,
$p(\e_{mj}|\z_{mj},\mu_{\z_{mj}}, \Sigma_{\z_{mj}})=[\mathcal{N}(\e_{mj};\muu_{mj},\Si_{mj})]^{w_{mj}}$, where $\mathcal{N}(\cdot)$ is the Gaussian distribution. Then we derive the third term $\sum_{j=1}^{J_m}\mathbb{E_{q}}[\log p(\e_{mj}|\z_{mj},\mu_{\z_{mj}}, \Sigma_{\z_{mj}})]$ in~\eqnref{eq:elbo_copy} as follows:}
\begin{align}
    \mathbb{E}_q[\log p(\e_{mj}|\z_{mj},\mu_{\z_{mj}}, \Si_{\z_{mj}})]
    &=\sum_{k} \phi_{mjk}w_{mj}\log\mathcal{N}(\e_{mj}|\muu_k,\Si_k)\nonumber\\
    &=\sum_{k}\phi_{mjk} w_{mj} \{-\frac{1}{2}(\e_{mj}-\muu_k)^T\Si_k^{-1}(\e_{mj}-\muu_k) \nonumber\\&-\log[(2\pi)^{d/2} \vert \Si_k\vert^{1/2}]\}.
\end{align}

\textbf{{Expanding the ELBO to the Loss Function.}} We can expand the ELBO in~\eqnref{eq:elbo} of the main paper as:
\begingroup\makeatletter\def\f@size{10}\check@mathfonts
    \begin{align}\label{eq:elbo_app}
    \mathcal{L}(\gamm,\ph;\alp,\{\muu\}_{k=1}^K,\{\Si\}_{k=1}^K) =&  \log\Gam(\sum_{k=1}^K\alpha_k)-\sum_{k=1}^K\log\Gam(\alpha_k) + 
    \sum_{k=1}^K(\alpha_k-1)(\Psi(\gamm_k)-\Psi(\sum_{k'=1}^K\gamm_{k'}))\nonumber\\
    &+\sum_{j=1}^{J}\sum_{k=1}^K\phi_{jk}(\Psi(\gamm_k)-\Psi(\sum_{k'=1}^K\gamm_{k'}))\nonumber\\
    &+\sum_{j,k}\phi_{jk} w_{j} \{-\frac{1}{2}(\e_{j}-\muu_k)^T\Si_k^{-1}(\e_{j}-\muu_k)
    -\log[(2\pi)^{d/2} \vert \Si_k\vert^{1/2}]\} \nonumber\\
    &-\log \Gam(\sum_{k=1}^K \gamm_{j}) + \sum_{k=1}^K \log\Gam(\gamm_k) - \sum_{k=1}^K (\gamm_k-1)(\Psi(\gamm_k)-\Psi(\sum_{k'=1}^K\gamm_{k'}))\nonumber\\
    &-\sum_{j=1}^{J}\sum_{k=1}^K \phi_{jk}\log \phi_{jk}.
    \end{align}
\endgroup

\textbf{Definition and Interpretation of the Loss Function.} We can interpret the meaning of each term of ELBO as follows:
\begin{itemize}
    \item \textbf{Regularization Term for Document-Level Explanations.} The sum of the first and the fourth terms, namely ${\EB}_q[\log p(\tha_m|\alp)] - {\EB}_q[\log q(\tha_m)]$, is equal to $-KL(q(\tha_m)|p(\tha_m|\alp))$, which is the negation of KL Divergence between the variational posterior probability $q(\tha_m)$ and the prior probability $p(\tha_m|\alpha)$ of the topic proportion $\tha_m$ for document $m$. Therefore maximizing the sum of these two terms is equivalent to minimizing the KL Divergence $KL(q(\tha_m)|p(\tha_m|\alp))$; this serves as a regularization term to make sure the inferred $q(\tha_m)$ is close to its prior distribution $p(\tha_m|\alp)$.  

 \item {\textbf{Regularization Term for Word-Level Explanations.}} Similarly, the sum of the second and the last terms (ignoring the summation over the word index $j$ for simplicity), namely ${\EB_q}[\log p(z_{mj}|\tha_m)] - {\EB_q}[\log q(z_{mj})]$ is equal to $-KL(q(z_{mj})|p(z_{mj}|\tha_m))$, which is the negation of the KL Divergence between the variational posterior probability $q(z_{mj})$ and the prior probability $p(z_{mj}|\tha_m)$ of the word-level topic assignment $z_{mj}$ for word $j$ of document $m$. Therefore maximizing the sum of these two terms is equivalent to minimizing the KL Divergence $KL(q(z_{mj})|p(z_{mj}|\tha_m))$; this serves as a regularization term to make sure the inferred $q(z_{mj})$ is close to its `prior' distribution $p(z_{mj}|\tha_m)$. 

\item {\textbf{Likelihood Term to Indicate How Much FLM Information is Explained.}} The third term $\sum\nolimits_{j=1}^{J_m}{\EB_q} [\log p(\e_{mj}|z_{mj},\muu_{z_{mj}}, \Si_{z_{mj}})]$ is to maximize the log likelihood $p(\e_{mj}|z_{mj},\muu_{z_{mj}}, \Si_{z_{mj}})$ of every contextual embedding $\e_{mj}$ (for word $j$ of document $m$) conditioned on the inferred $z_{mj}$ and the parameters $(\muu_{z_{mj}}, \Si_{z_{mj}})$. 
\end{itemize}
{In this way, we expand the ELBO to a concrete loss function. Each line of~\eqnref{eq:elbo_app} corresponds to the expansion of each of the five terms in the ELBO mentioned above (i.e.,~\eqnref{eq:elbo} in the paper).}

\section{Experimental Settings and Implementation Details}\label{app:experiment}
We will release all code, models, and data. Below we provide more details on the experimental settings and practical implementation.

\textbf{Data Preprocessing and More Datasets.} 
\begin{table}[t]
      \caption{Dataset statistics, including the number of documents ($M$), vocabulary size ($V$), the number of corpus categories ($L$), and the average document length ($\overline{J}$).}\label{table:dataset}
      \vskip -0.3cm
        \footnotesize
      \centering
      \scriptsize
      \begin{tabular}{lcccc}
        \toprule
        \textbf{Dataset}  & $M$ & $V$  & $L$  & $\overline{J}$   \\
       \midrule
        20 Newsgroups  & 16,309 &   1,612  &  20 & 48\\
        M10 & 8,355 & 1,696 & 10 & 5.9 \\
        BBC News & 2,225 & 2,949 & 5&  120  \\
        
        \bottomrule
      \end{tabular}
       \vskip -0.5cm
\end{table}
{We follow \citet{terragni2021octis} and \citet{zhang2022cetoic} to pre-process these datasets. The statistics of the datasets are summarized in~\tabref{table:dataset}. We use the standard 8:1:1 train/validation/test set split.} We {also} use the GLUE benchmark~\citep{wang2018glue} to perform \emph{additional} conceptual interpretation in this section and~\secref{app:interpret_more}. This benchmark includes multiple sub-tasks of predictions, with the paired sentences as inputs. In this paper, we use 4 datasets from GLUE (MRPC, RTE, STS-B, and QQP) to show contextual interpretations. {Specifically, we apply VALC to multiple complex natural language understanding (NLU) tasks in the GLUE benchmark. For example, in~\appref{app:interpret_more}, we show the three-level conceptual explanations of \emph{four different tasks} in the GLUE benchmark using VALC, i.e., 
\begin{itemize}[nosep]
    \item \textbf{Microsoft Research Paraphrase Corpus (MRPC)}, where the task is paraphrase identification and semantic textual similarity, 
    \item \textbf{Recognizing Textual Entailment (RTE)}, where the task is to determine whether one sentence (the premise) entails another sentence (the hypothesis), 
    \item \textbf{Semantic Textual Similarity Benchmark (STS-B)}, where the task is to measure the degree of semantic similarity between pairs of sentences (from 0 to 5), and 
    \item \textbf{Quora Question Pairs (QQP)}, where the task is to classify whether one question is the duplicate of the other. 
\end{itemize}
}

\textbf{Implementation.} 
{We implemented and trained the model using PyTorch~\cite{paszke19pytorch} on an A5000 GPU with 24GB of memory. The training duration was kept under a few hours for all datasets. We utilized the Adam optimizer~\cite{kingma14adam} with initial learning rates varying between $10^{-5}\sim10^{-3}$, tailored to the specific requirements of each dataset.}

\textbf{Visualization Postprocessing.} 
For better showcase the dataset-level concepts as in~\figref{topic-m10} of the main paper, we may employ simple linear transformations on the embedding of words after the aforementioned PCA step, in order to scatter all the informative words on the same figures. However, for some datasets such as STS-B, this is not necessary; therefore we do not use it for these datasets.  

\textbf{Topic (Concept) Identification.} 
Inspired by~\citet{blei2003latent}, we identify meaningful topics by listing the top-5 topics for each word, computing the inverse document frequency (IDF), and filtering out topics with the lowest IDF scores. Note that although GLUE benchmark are datasets that consists of documents with small size, making it particularly challenging for traditional topic models (such as LDA) to learn topics; interestingly our VALC can still perform well in learning the topics. We contribute this to the following observations: (1) Compared to traditional LDA using \emph{discrete} word representations, VALC uses \emph{continuous} word embeddings. In such a continuous space, topics learned for one word can also help neighboring words; this alleviates the sparsity issue caused by short documents and therefore learns better topics. (2) VALC's attention-based continuous word counts further improves sample efficiency. In VALC, important words have larger attention weights and therefore larger continuous word counts. In this case, \emph{one} important word in a sentence possesses statistical (sample) power equivalent to \emph{multiple} words; this leads to better sample efficiency in VALC.

\textbf{Computational Complexity.} 
{Our VALC introduces minimal overhead in terms of model training cost. Specifically, VALC's computational complexity is $O(TKd^2)$, where $T$ is the number of epochs (a small number, such as 3, is sufficient for convergence), $K$ is the number of concepts, and $d$ is the dimension of the embeddings (in hidden layers). This means that VALC's computational cost \emph{scales linearly} with the number of concepts $K$ (similar to existing methods). }

\textbf{More NLP Tasks.} 
{VALC can be naturally applied to other NLP tasks, such as named entity recognition (NER), reading comprehension, or question answering. Specifically, these tasks involve transformer predictions from multiple positions within the context, rather than relying solely on the `CLS' token. For example, NER predicts each token in the document as the beginning (`B') of an entity, the inside (`I') of entities, etc.
To accommodate this and use VALC to explain each token $j$ in the context, we can substitute the attention from the `CLS' token with (1) the attention from the `CLS' token to all tokens of the previous layer with (2) the attention from token $j$ to all tokens of the previous layer in transformers (e.g., using the attention weights from the predicted label `B' to all tokens of the previous layer as $\a_{m}$ in VALC). This adaptation allows VALC to maintain its explanatory power across various NLP applications, demonstrating its versatility and effectiveness in a wide range of tasks.}

\section{{More Details on Concept Editing}} \label{app:concept_edit}
{We perform concept pruning to the CLS embeddings for VALC (details in~Alg. \ref{alg:prune_valance}). Since BERTopic and CETopic can infer concepts (topics) only at the document level, their only choice is to prune a concept by completely removing input tokens assigned to the concept (as mentioned in~\secref{sec:setup} and \ref{sec:exp_property}). To compare our learned concepts with the baseline models, we first follow their configurations~\citep{grootendorst2020bertopic,zhang2022cetoic} to fix BERT model parameters when learning the topics/concepts, train a classifier on top of the fixed contextual embeddings, and then perform concept pruning~\citep{koh2020CBM} for different evaluated models on the same classifier. Note that concept editing is deterministic; therefore, we conduct our experiments with a single run.}

{Specifically, we} assume each BERT model contains a backbone and a classifier. To perform concept editing: 
\begin{itemize}[nosep,leftmargin=18pt]
 \item[(1)]We first train a classifier on top of the \emph{fixed} BERT embeddings generated by the \emph{fixed} backbone to get the original accuracy in the `Unedited' column (in~\tabref{table:prune_dataset} and~\tabref{table:prune_ablation} of the main paper). 
\item[(2)] We then apply the same embedding cluster methods to these BERT embeddings to infer the concepts/topics for each dataset.
\item[(3)] Finally, with the inferred concepts/topics from the baselines (SHAP/LIME, BERTopic and CETopic in~\tabref{table:prune_dataset} of the main paper) and our VALC variants (Unweighted and Weighted in~\tabref{table:prune_ablation} of the main paper), we perform concept editing and feed the concept-edited embeddings into the trained classifier from Step (1) to compute the editing accuracy for different methods. 
\end{itemize}
Since here one \emph{does not fully finetune the BERT model} (i.e., keeping the backbone fixed), the editing accuracy is expected to be lower than the `Finetune' column (in~\tabref{table:prune_dataset} and~\tabref{table:prune_ablation} of the main paper), which serves as the oracle.~\tabref{table:prune_dataset} of the main paper shows that our VALC learns better concepts than the baselines, and~\tabref{table:prune_ablation} of the main paper shows that the weighted variant of VALC performs better. 
\begin{algorithm}[H]
 \caption{Algorithm for VALC Document-Level Concept Editing }\label{alg:prune_valance_cls}
 \textbf{Input:} FLM $f(\cdot)$, classifier $g(\cdot)$, classification loss $L$, dataset $\{\DM_m\}_{m=1}^{M}$, labels $\y$, constant factor $\omega$.\\
 \textbf{for} {$m=1:M$} \textbf{do}{

 \quad $\c_{m} = f(\DM_{m})$ 

 
\quad $\x^* = QP(\c_{m}, \{\muu_k\}_{k=1}^{K})$ 

 \quad $k^* = \arg \min_{k=1}^K L(g(\c_{m} - \omega\cdot x^*_k \muu_{k}), y_m)$

 \quad $\c_{m} \leftarrow\c_{m} - \omega\cdot x^*_{k^*} \muu_{k^*}$
 }
\end{algorithm}
{Note that SHAP and LIME both interpret the CLS token's embedding, and hence their concept vectors have the same dimension as the FLM embedding vector (768 in our case). When we conduct concept editing on the $k$'th dimension/concept, we simply subtract the CLS embedding's dimension $k$ with the average value in the batch on dimension $k$ (which means that we know little about the concept/dimension $k$ on this document), and keep values of the other dimensions unchanged. Note that the pruning process is exactly the same for SHAP and LIME. Therefore SHAP and LIME have identical test accuracy and accuracy gain.}

\textbf{{Document-Level Concept Editing.}} \label{app:document_edit}
{We describe the document-level concept editing algorithm of VALC in~\algref{alg:prune_valance_cls}. $c_m$ denotes the `CLS' embedding of document $m$ (see~\figref{fig:overview} of the main paper).}

\section{{Connections Between the Defined Properties and Empirical Results}}
VALC is able to show which words or embeddings contributed to the document-level concept $k$. Specifically, our variational parameter (a vector) $\ph_{mj}\in \mathbb{R}^{K}$ describes how much word $j$ contributes to document $m$. For example, the $k$-th entry of $\ph_{mj}$, denoted as $\phi_{mjk}$ in the paper, describes how much word $j$ contributes to document $m$ in terms of concept $k$. Therefore, one could use $\arg\max_j \phi_{mjk}$ to find the word that contributes most to document $m$'s concept $k$. 
Below, we will explain these four properties using~\figref{topic-m10} as a running example. 
\begin{itemize}
    \item[(1)] \textbf{Multi-Level Structure} ensures that VALC learns the dataset-, document-, and word-level concepts jointly. In \figref{topic-m10}:
\begin{itemize}
\item \textbf{Dataset-level} concepts are highlighted by the top words of each concept (the top right box of \figref{topic-m10}) and the distribution of their embeddings in the FLM (left and middle figures of \figref{topic-m10}); for example, \emph{Concept 5 (data analysis)} is marked in red.
\item \textbf{Document-level} concepts are demonstrated by each document's topic; for instance, in the box for document (a) in \figref{topic-m10} (right), VALC identifies \emph{Topic (Concept) 5} as key to the FLM's prediction of the label 3 (computer science). 
\item \textbf{Word-level} concepts are identified by words in documents. For example, in the box for document (a) in \figref{topic-m10}, VALC highlights the words `genetic' and `neural' because they are highly related to \emph{Concept 5 (data analysis)}. Terms like `genetic algorithms' and `neural networks' are related to data analysis, aligning with the document-level concept.
    \end{itemize}
    \item[(2)] \textbf{Normalization} ensures that concept learning is regulated and smoothed, with inferred concepts appearing reasonable. Specifically, in the document-level explanation $\tha_m$ and word-level explanation $\ph_{mj}$, all concepts are assigned a value within the range of $0\sim 1$, and all entries sum up to 1, i.e., $\sum_{k=1}^{K}\phi_{mijk}=1$ and $\sum_{k=1}^K\theta_{mk}=1$. This introduces `competition' among different concepts; a larger strength for one concept means smaller strength for other concepts. Therefore, together with the help of the Dirichlet prior, it implicitly encourages sparser concept-level explanations $\theta_m$, which are more aligned with humans' cognitive processes and more human-understandable (humans tend to make decisions with a \emph{small} set of concepts). 
    \item[(3)] \textbf{Additivity} enables FLMs to incorporate relevant concepts and exclude irrelevant ones, thereby enhancing prediction accuracy (as shown in~\tabref{table:prune_dataset} and~\tabref{table:prune_ablation}). For example, in document (a) of \figref{topic-m10}, VALC identifies \emph{Concept 5} as a highly related concept, distinguishing it from less related concepts. In practice, this may help practitioners identify key concepts in model prediction and more effectively intervene to improve model prediction accuracy (e.g., an expert may find that a concept is relevant and manually down-weight the concept to enhance the model's prediction).
    \item[(4)] \textbf{Mutual Information Maximization} ensures a strong correlation between (1) VALC's generated concept explanations and (2) the explained model's representation and predictions. In other words, it ensures that VALC is explaining the target FLM, rather than generating concept explanations irrelevant to the target FLM. For instance, in document (a) of \figref{topic-m10}, the inferred document-level \emph{Concept 5} (data analysis) effectively explains the FLM prediction, i.e., label 3 (computer science), by highlighting the intrinsic link between the data analysis concept and the class label computer science. This connection is evidenced by the words in dataset-level \emph{Concept 5} (top right box). The mutual information between the inferred \emph{Concept 5} (data analysis) and label 3 (computer science) contributes to generating high-quality explanations.

\end{itemize}

\section{More Conceptual Interpretation Results in Different Downstream Tasks}\label{app:interpret_more}

\textbf{Dataset-Level Interpretations.} \label{sec:dataset-interpret}
As in the main paper, we leverage VALC as an interpreter on MRPC, RTE, STS-B and QQP, respectively, sample $3,3,4,4$ concepts (topics) for each dataset respectively, and plot the word embeddings of the top words (closest to the center $\muu_i$) in these concepts using PCA. \figref{topic}(left) shows the concepts from MRPC. We can observe \red{Concept 20} is mostly about \underline{policing}, including words such as `suspect', `police', and `house'. \green{Concept 24} is mostly about \underline{politics}, including words such as `capital', `Congress', and `Senate'. \purple{Concept 27} contains mostly \underline{names} such as `Margaret' and `Mary'. Similarly, \figref{topic}(right) shows the concepts from RTE. We can observe \red{Concept 67} is related to \underline{West Asia} and includes words such as `Quran' and `Pasha'. \green{Concept 13} is related to \underline{Europe} and includes European countries/names such as `Prussia' and `Salzburg'. \purple{Concept 91} is mostly about \underline{healthcare} and includes words such as `physiology' and `insulin'. \figref{topic-stsb} shows the concepts from STS-B. We can observe \red{Concept 63} is mostly about \underline{household and daily life}, including words such as `trash', `flowers', `airs', and `garden'. \purple{Concept 60} is mostly about \underline{tools}, including words such as `stations', `rope', `parachute', and `hose'. \blue{Concept 84} is mostly about \underline{national security}, including words such as `guerilla', `NSA', `espionage', and `raided'. \green{Concept 55} contains mostly \underline{countries and cities} such as `Kiev', `Moscow', `Algeria', and `Ukrainian'. 
Similarly, \figref{topic-qqp} shows the concepts from QQP. We can observe that \red{Concept 12} is mostly about \underline{negative attitude}, including words such as `boring', `criticism', and 'blame'. \purple{Concept 73} is mostly about \underline{Psychology}, including words such as `adrenaline', `haunting', and 'paranoia'. \blue{Concept 34} is mostly about \underline{prevention and conservatives}, including words such as `destroys', `unacceptable', and 'prohibits'. \green{Concept 64} is mostly about \underline{strategies}, including words such as `rumours', `boycott', and 'deportation'. 
\begin{figure*}[!t]
        \centering
        \includegraphics[width = 1.0\textwidth]{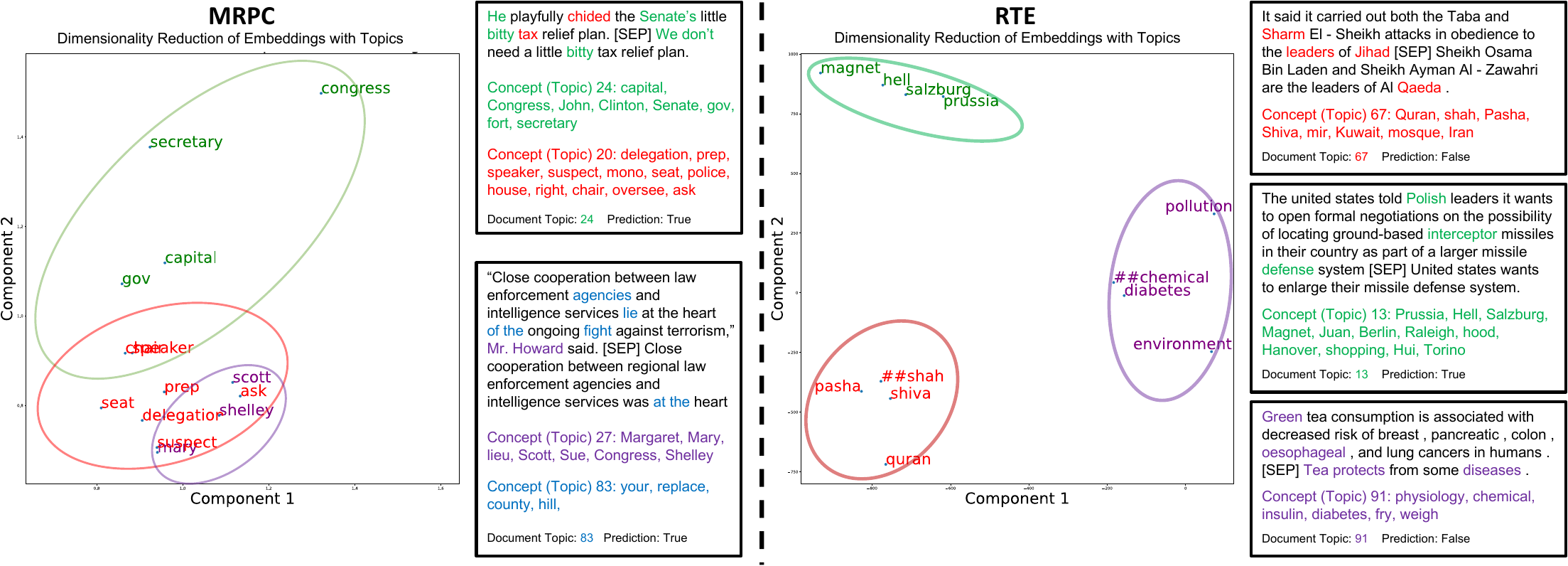}
        \vskip -0.2cm
        \caption[width = 0.8\linewidth]{Visualization of VALC's learned topics of contextual word embeddings. \textbf{Left:} MRPC's dataset-level interpretation with two example documents. \blue{Concept 83} is relatively far from the other three concepts in the embedding space; therefore we omit it on the left panel for better readability. 
        \textbf{Right:} RTE's dataset-level interpretation with three example documents. }
        \label{topic}
\end{figure*}

\begin{figure*}[!t]
        \centering
        \includegraphics[width = 1.0\textwidth]{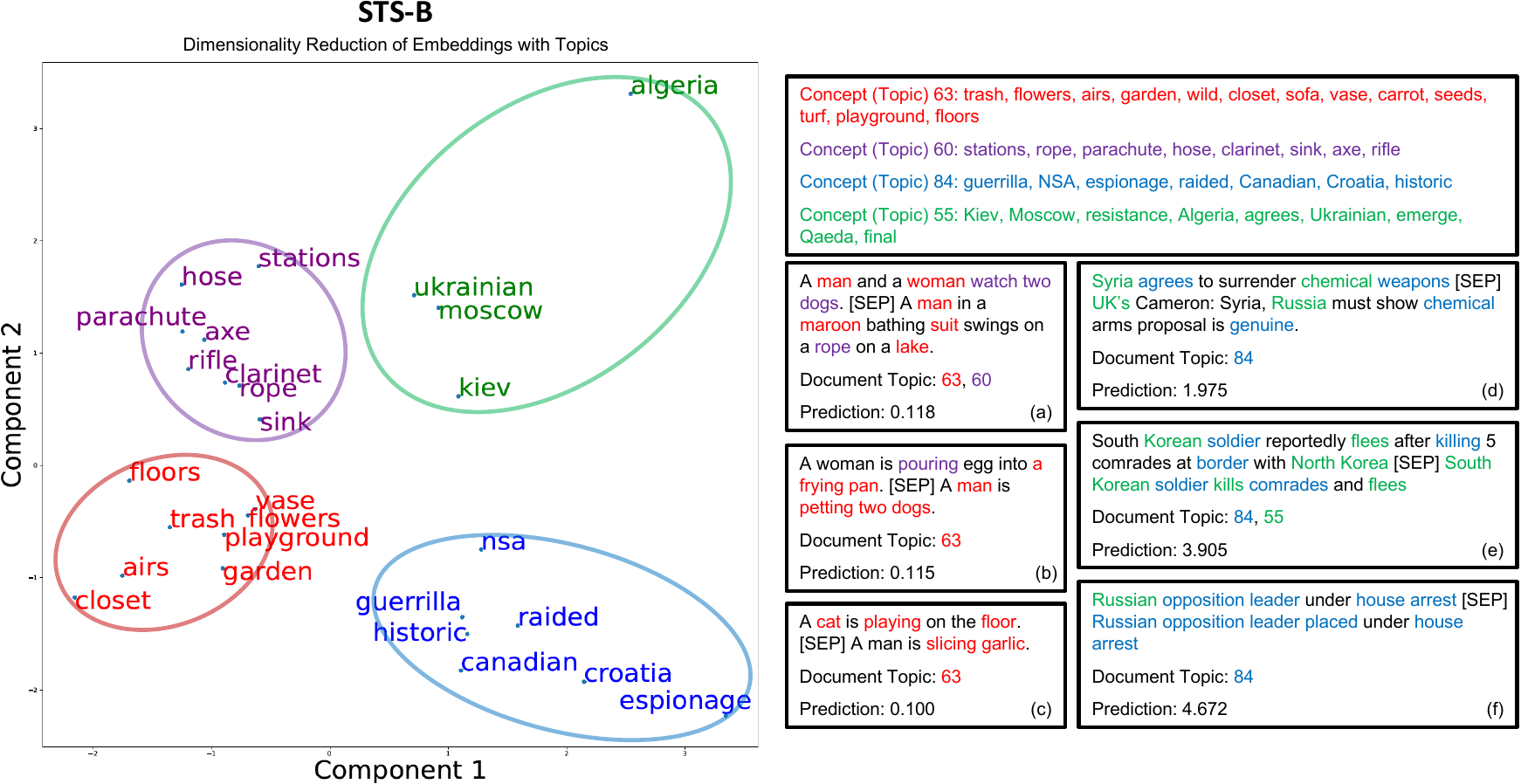}
        \vskip -0.2cm
        \caption[width = 0.5\textwidth]{Visualization of VALC's learned topics of contextual word embeddings. We show STS-B's dataset-level interpretation with six example documents. The prediction of VALC is between the range of $[0,5]$.}
        \label{topic-stsb}
\end{figure*}

\begin{figure*}[!t]
        \centering
        \includegraphics[width = 1.0\textwidth]{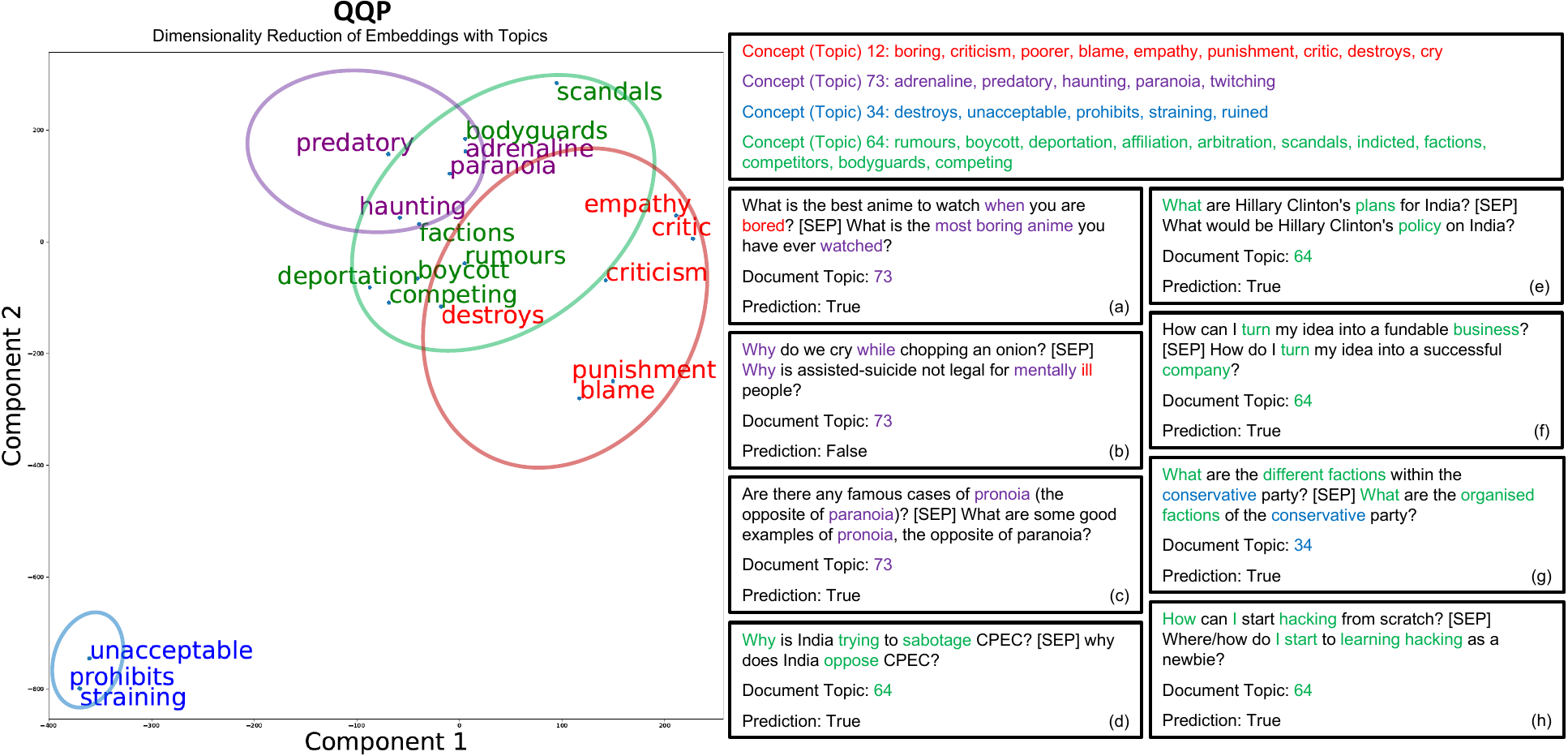}
        \vskip -0.2cm
        \caption[width = 0.5\textwidth]{Visualization of VALC's learned topics of contextual word embeddings. We show QQP's dataset-level interpretation with eight example documents.}
        \label{topic-qqp}
        \vskip -0.6cm
\end{figure*}

\textbf{Document-Level Interpretations.} 
For document-level conceptual interpretations, we sample two example documents from MRPC (\figref{topic}(left)), three from RTE (\figref{topic}(right)), six from STS-B (\figref{topic-stsb}) and eight from QQP (\figref{topic-qqp}), respectively, where each document contains a pair of sentences. The MRPC task is to predict whether one sentence paraphrases the other. For example, in the first document of MRPC, we can see that our VALC correctly interprets the model prediction `True' with \green{Concept 24 (politics)}. The RTE task is to predict whether one sentence entail the other. For example, in the second document of RTE, VALC correctly interprets the model prediction `True' with \green{Concept 13 (countries)}. The STS-B task is to predict the semantic similarity between two sentences with the score range of $[0,5]$. For example, in Document~(a) of \figref{topic-stsb}, we can see that VALC correctly interpret the model's predicted similarity score `$0.118$' (which is relatively low,) with \red{Concept 63 (household and daily life)} and \purple{Concept 60 (tools)}. Similarly, in Document~(f) of \figref{topic-stsb}, we can see that VALC correctly interpret the model's predicted similarity score `$4.672$' (which is relatively high) with \blue{Concept 84 (national security)}. 
The QQP task is to predict whether the two questions are paraphrase of each other. For example, in Document~(b) of \figref{topic-qqp}, we can see that VALC correctly interprets the model's predicted label `False' with \purple{Concept 73 (Psychology)}. Similarly, in Document~(e) of \figref{topic-qqp}, we can see that VALC correctly interprets the model's predicted label `True' with \green{Concept 64 (strategies)}.

\textbf{Word-Level Interpretations.}
For word-level conceptual interpretations, we can observe that VALC interpret the FLM's prediction on MRPC's first document (\figref{topic}(left)) using words such as `senate' and `bitty' that are related to politics. Note that the word `bitty' is commonly used (with `little') by politicians to refer to the small size of tax relief/cut plans. Similarly, for RTE's first document (\figref{topic}(right)), VALC correctly identifies \red{Concept 67 (West Asia)} and interprets the model prediction `False' by distinguishing between keywords such as `Jihad' and `Al Qaeda'. likewise, we can observe that VALC interprets FLM's prediction on Document~(c) of \figref{topic-stsb} using words such as `cat', `floor', and `garlic' that are related to \red{household and daily life}. Also, VALC interprets FLM's prediction on Document~(e) of \figref{topic-stsb} using words such as `soldier' and `border' that are related to \blue{national security}. Similarly, for QQP's Document~(d) (\figref{topic-qqp}), VALC correctly interprets the model prediction `True' by identifying keywords such as `sabotage' and `oppose' with similar meanings in the topic of \green{strategies}. For QQP's Document~(g), (\figref{topic-qqp}), VALC interprets the words in the both sentences with the same semantics, such as `conservative' that is related to \blue{prevention and conservatives} (note that in politics, `conservative' refers to parties that tend to prevent/block new policies or legislation), and thereby predicting the correct label `True'.

\textbf{Example Concepts.} 
\begin{table}[t]
\setlength{\tabcolsep}{2pt}
      \caption{Example concepts on RTE dataset learned by VALC.}\label{table:concept-example-rte}
        \footnotesize
      \centering
 
      \resizebox{0.9\textwidth}{!}{%
      \begin{tabular}{l|cccccccc}

        \midrule
        {\bf{Concepts}} & \multicolumn{8}{c}{Top Words}      \\
       \midrule
        \bf{bio-chem}  & cigarette & biological  &  ozone & cardiovascular & chemist  & liver & chemical& toxin\\ 
        \bf{citizenship} & indies & bolivian & fiji & surrey & jamaican & dutch & latino & caribbean\\ 
       \bf{names} &  mozart & spielberg & einstein & bush & kurt & liszt & hilton & lynn\\ 
        \bf{conspiracy} & secretly & corrupt & disperse & infected & ill & hidden & illegally & sniper \\ 
        \bf{administration} & reagan & interior & ambassador & prosecutor & diplomat & legislative & spokesman & embassy \\ 
        \bf{crime} & fraud & laundering & sheriff & prosecutor & corruption & fool&robber & greed \\ 
        \bottomrule
      \end{tabular}
      }
       \vskip -0.0cm
\end{table}
Following~\citet{blei2003latent}, we show the learned concepts on the RTE dataset in~\tabref{table:concept-example-rte}, which is complementary to aforementioned explanations. We select several different topics from~\figref{topic}. As in Sec. 5.4 of the main paper, we obtain top words from each concept via first calculating the average of the each word's corresponding contextual embeddings over the dataset, and then getting the nearest words to each topic center ($\muu_k$) in the embedding space. As we can see in~\tabref{table:concept-example-rte}, VALC can capture various concepts with profound and accurate semantics. Therefore, although FLM embeddings are contextual and continuous, our VALC can still find conceptual patterns of words on the dataset-level.  

\section{{More Quantitative Results.}}\label{app:quantitative}

\textbf{{Document Classification with VALC Concepts.}}
We conducted additional experiments to perform document classification using the `CLS' token's embedding and $\tha$ (inferred from VALC) as features.~\tabref{tab:theta_comparison} shows the results on three datasets. The results show that our VALC can learn meaningful concept vector $\tha$, which can improve model predictions of document labels.

\begin{table}[t]
  \centering
  \caption{Comparison of Unedited and Unedited+$\tha$ on 20 Newsgroups, M10, and BBC News. We mark the best results with \textbf{bold face}.}
  
    \begin{tabular}{l|rr}
    \hline
    \textbf{} & {Unedited} & {Unedited+$\tha$} \\
    \hline
   {20 Newsgroups} & 51.26 & \textbf{51.74} \\
    \hline
   {M10} & 69.74 & \textbf{70.76} \\
    \hline
   {BBC News} & 93.72 & \textbf{94.90} \\
    \hline
    \end{tabular}%
  \label{tab:theta_comparison}%
\end{table}%

\section{Theory on the Mutual Information Maximization Property}
\label{app:proof}
We provide the following proof of Theorem 4.1 of the main paper. 

For convenience, let $\Omega = (\muu_{k=1}^K, \Sigma_{k=1}^K)$, and $\beta = (\tha_m,\z_m)$. 

We then introduce a helper joint distribution of the variables $\e_m$ and $\beta$, $s(\e_m, \beta)=p(\e_m)q(\beta|\e_m)$.

According to the definition of ELBO of Section 3.4.1, in~\eqnref{eq:MI_bound}, we have 

\begin{align}
LHS = \LM(\gamma_m, \phi_m; \alpha, \Omega) = \EB_{p(\e_m)}[\EB_{q(\beta)}[\log p(\e_m|\Omega, \beta)]] + \EB_{q(\beta)}[\log q(\beta|\Omega)].
\end{align}
Since $\EB_{q(\beta)}[\log q(\beta|\Omega)] \le 0$, we only need to prove that 
\begin{align}
\EB_{p(\e_m)}[\EB_{q(\beta)}[\log p(\e_m|\Omega, \beta)]] \le I_s(\e_m;\beta) - H(\e_m) = RHS.
\end{align}
Then we have that
\begingroup\makeatletter\def\f@size{7.5}\check@mathfonts
\begin{align}
\EB_{p(\e_m)}[\EB_q[\log p(\e_m|\beta,\Omega)]] &\le 
\EB_{p(\e_m)}[\EB_q[\log p(\e_m|\beta)]] \nonumber\\ 
&=\EB_{p(\e_m)}[\EB_q[\log\frac{q(\e_m|\beta)}{p(\e_m)}\frac{p(\e_m)p(\e_m|\beta)}{q(\e_m|\beta)}]]\nonumber\\
&=\EB_{p(\e_m)}[\EB_q[\log\frac{q(\e_m|\beta)}{p(\e_m)}]]
+\EB_{p(\e_m)}[\EB_q[\log p(\e_m)]] + \EB_{p(\e_m)}[\EB_{q}[\log\frac{p(\e_m|\beta)}{q(\e_m|\beta)}]]\nonumber\\
&=I_s(\e_m;\beta) - H(\e_m) - \EB_{q}[KL(q(\e_m|\beta)|p(\e_m|\beta))]\nonumber\\
&\le I_s(\e_m;\beta) - H(\e_m) - 0 = RHS, 
\end{align}
\endgroup
which concludes the proof of Theorem 4.1. 

\section{Theoretical Analysis on Continuous Word Counts}
\label{app:theory}

Before going to the claims and proofs, first we specify some basic problem settings and assumptions. Suppose there are $K+1$ topic groups, each of which is regarded to be sampled from a parameterized multivariate Gaussian distribution. In specific, the $K+1$ 'th distribution of topic has a much larger covariance, and in the same time, closed to the center of embedding space. The prementioned properties can be measured by a series of inequalities:

The approximate marginal log-likelihood of word embeddings, i.e., the third term of the ELBO as mentioned in Eqn. 2 of the main paper, is: 
\begin{align} 
\label{marginal likelihood}
   \mathcal{L}^{(train)} &= \sum\nolimits_{j=1}^{J_m}\EB_q[\log p(\e_{mj}|z_{mj},\muu_{z_{mj}}, \Si_{z_{mj}})]\nonumber\\
    &=\sum_{m,j,k}\phi_{mjk} w_{mj} \{-\frac{1}{2}(\e_{mj}-\muu_k)^T\Si_k^{-1}(\e_{mj}-\muu_k)    -\log[(2\pi)^{d/2} \vert \Si_k\vert^{1/2}]\}.
\end{align}
{The above equation is the training objective, yet for fair comparison of different training schemes, we calculate the approximated likelihood with word count $1$ for all words.}

\begin{align}
\label{marginal eval}
   \mathcal{L}^{(eval)} &= \sum\nolimits_{j=1}^{J_m}\EB_q[\log p'(\e_{mj}|z_{mj},\muu_{z_{mj}}, \Si_{z_{mj}})]\nonumber\\
    &=\sum_{m,j,k}\phi_{mjk} \{-\frac{1}{2}(\e_{mj}-\muu_k)^T\Si_k^{-1}(\e_{mj}-\muu_k)    -\log[(2\pi)^{d/2} \vert \Si_k\vert^{1/2}]\}.
\end{align}

\subsection{Gaussian Mixture Models}

Suppose we have a ground truth GMM model with parameters $\pii^*\in \mathbb{R}^K$ and $\{\muu_k^*, \Si_k^*\}_{k=1}^K$, with $K$ different Gaussian distributions. In the dataset, let $N$ and $N_s$ denote the numbers of non-stop-words and stop-words, respectively. Then the marginal log likelihood of a learned GMM model on a given data sample $\e$ can be written as 
\begin{align}
    p(\e|\{\muu,\Si\},\pii) = \sum_{k=1}^{K} \pii_k\mathcal{N}(\e;\muu_k,\Si_k).
\end{align}
Assuming a dataset of $N+N_s$ words $\{\e_i\}_{i=1}^{N+N_s}$ and taking the associated weights $w_{i}$ for each word into account, the log-likelihood of the dataset can be written as
\begingroup\makeatletter\def\f@size{6.5}\check@mathfonts
\begin{align}
    \sum_{i=1}^{N+N_s} p(\e_i|\{\muu_k,\Si_k\}_{k=1}^K,\pii) =  \sum_{i=1}^{N}\log\sum_{k=1}^{K}  w_i \pii_k\mathcal{N}(\e_i;\muu_k,\Si_k) + \sum_{i=N+1}^{N+N_s}\log\sum_{k=1}^{K}  w_i \pii_k\mathcal{N}(\e_i;\muu_k,\Si_k).
\end{align}
\endgroup
Leveraging Jensen's inequality, we obtain a lower bound of the above quantity (denoting as $\Tha$ the collection of parameters $\{\muu_k, \Si_k\}_{k=1}^K$ and $\pii$):
\begingroup\makeatletter\def\f@size{6.5}\check@mathfonts
\begin{align}
    \mathcal{L}_{\text{GMM}}(\Tha,\{w_i\}) = \sum_{i=1}^{N} w_i\log\sum_{k=1}^{K}  \pi_k\mathcal{N}(\e_i;\muu_k,\Si_k) + \sum_{i=N+1}^{N+N_s} w_i\log\sum_{k=1}^{K} \pi_k\mathcal{N}(\e_i;\muu_k,\Si_k) + C,
\end{align}
\endgroup
where C is a constant. 

In the following theoretical analysis, we consider the following three different configurations of the weights $w_i$.

\begin{definition}[\textbf{Weight Configurations}]
\label{def:three_weights}
We define three different weight configurations as follows: 
\begin{itemize}
    \item Identical Weights: $w_i = \frac{1}{N+N_s}$, $i\in \{1,2,\dots,N+N_s\}$
    \item Ground-Truth Weights : $w_i =\begin{cases}  \frac{1}{N},~~i\in\{1,2,\dots,N\}\\ 0,~~i\in \{N+1,N+2,\dots,N+N_s\}\end{cases}$
    \item Attention-Based Weights: $w_i =\begin{cases} \lambda_1 \in  [\frac{1}{N+N_s},\frac{1}{N}], ~~i\in\{1,2,\dots,N\}\\ \lambda_2 \in [0,\frac{1}{N+N_s}],~~i\in \{N+1,N+2,\dots,N+N_s\}\end{cases}$
\end{itemize}
\end{definition}

\begin{definition}[\textbf{Advanced Weight Configurations}]
\label{def:three_free_weights}
We define three different weight configurations as follows: 
\begin{itemize}
    \item Identical Weights: $w_i = \frac{1}{N+N_s}$, $i\in \{1,2,\dots,N+N_s\}$
    \item Ground-Truth Weights : $w_i =\begin{cases}  \frac{1}{N},~~i\in\{1,2,\dots,N\}\\ 0,~~i\in \{N+1,N+2,\dots,N+N_s\}\end{cases}$
    \item Attention-Based Weights: $w_i \in\begin{cases}  [\frac{1}{N+N_s},\frac{1}{N}], ~~i\in\{1,2,\dots,N\}\\ [0,\frac{1}{N+N_s}],~~i\in \{N+1,N+2,\dots,N+N_s\}\end{cases}$
\end{itemize}
\end{definition}

\begin{definition}[\textbf{Optimal Parameters}]
\label{def:opt_para}
With~\defref{def:three_weights}, the corresponding optimal parameters are then defined as follows:
\begin{align}
    & \Tha_I = \arg \max_{\Tha} \mathcal{L} (\Tha; \w\rightarrow \mbox{Identical}),\\
    & \Tha_G = \arg \max_{\Tha} \mathcal{L} (\Tha; \w\rightarrow \mbox{GT}),\\
    & \Tha_A = \arg \max_{\Tha} \mathcal{L} (\Tha; \w\rightarrow \mbox{Attention}),
\end{align}
where $\w\rightarrow \mbox{Identical}$, $\w\rightarrow \mbox{GT}$, and $\w\rightarrow \mbox{Attention}$ indicates that `Identical Weights', `Ground-Truth Weights', and `Attention-Based Weights' are used, respectively. 
\end{definition}
\begin{lemma}\label{lem:cmp_lemma}
Suppose we have two series of functions $\{f_{1,i}(x)\}$ and $\{f_{2,i}(x)\}$,  with two non-negative weighting parameters $\lambda_1, \lambda_2$ satisfying $N\lambda_1 + N_s\lambda_2 = 1$. We define the final objective function $f(\cdot)$ as:
\begin{align}
    f(x;\lambda_1, \lambda_2) = \lambda_1 \sum_{i=1}^{N} f_{1,i}(x) + \lambda_2 \sum_{i=N+1}^{N_s} f_{2,i}(x).
\end{align}
We assume two pairs of parameters $(\lambda_1, \lambda_2)$ and 
$(\lambda'_1, \lambda'_2)$, where
\begin{align}
    \lambda_1 \ge \lambda'_1,  \label{eq:lambda_1}\\
   \lambda_2 \le \lambda'_2 . \label{eq:lambda_2}
\end{align}

Defining the optimal values of the objective function for different weighting parameters as
\begin{align}
\label{x1_def}
    & \hat x = \arg \max_x f(x;\lambda_1,\lambda_2), \\
\label{x2_def}
    & \hat x' = \arg \max_x f(x;\lambda'_1,\lambda'_2),
\end{align}
we then have that
\begin{align}
     f(\hat x;\frac{1}{N},0) \ge f(\hat x';\frac{1}{N},0).
\end{align}
\end{lemma}

\begin{proof}

We prove this theorem by contradiction. Suppose that we have
\begin{align}\label{contradict}
    f(\hat x;\frac{1}{N},0) < f(\hat x';\frac{1}{N},0).
\end{align}

According to \eqnref{eq:lambda_1}, i.e., $\lambda_1\ge \lambda'_1$, and the equation $N\lambda_1+N_s\lambda_2=1$, we have
\begin{align}
    & \lambda_1 \lambda'_2 = \lambda_1 \frac{1-N\lambda'_1}{N_s} \ge \lambda'_1 \frac{1-N\lambda_1}{N_s} = \lambda'_1 \lambda_2.
\end{align}

According to~\eqnref{x2_def}, we have the following equality:
\begin{align}
    f(\hat x;\lambda'_1,\lambda'_2) \le f(\hat x';\lambda'_1,\lambda'_2).
\end{align}
Combined with the aforementioned assumption in~\eqnref{contradict}, we have that
\begin{align}
   &\lambda'_2 f(\hat x;\lambda_1,\lambda_2) =\lambda_1\lambda'_2\sum_{i=1}^{N} f_{1,i}(\hat x) +   \lambda_2\lambda'_2\sum_{i=N+1}^{N_s} f_{2,i}(\hat x) \\
   =& (\lambda'_1\lambda_2\sum_{i=1}^{N} f_{1,i}(\hat x) +   \lambda'_2\lambda_2\sum_{i=N+1}^{N_s} f_{2,i}(\hat x) )  
   + (N (\lambda_1\lambda'_2-\lambda'_1\lambda_2)\cdot \frac{1}{N}\sum_{i=1}^{N} f_{1,i}(\hat x) )  \\
   =&{\lambda_2}f(\hat x;\lambda'_1,\lambda'_2) + N(\lambda_1\lambda'_2-\lambda'_1\lambda_2)f(\hat x;\frac{1}{N},0) \\
    <& {\lambda_2}f(\hat x';\lambda'_1,\lambda'_2) + N(\lambda_1\lambda'_2-{\lambda'_1\lambda_2})f(\hat x';\frac{1}{N},0) \\
    =& (\lambda'_1\lambda_2\sum_{i=1}^{N} f_{1,i}(\hat x') +   \lambda'_2\lambda_2\sum_{i=N+1}^{N_s} f_{2,i}(\hat x') )  
   + (N (\lambda_1\lambda'_2-\lambda'_1\lambda_2)\cdot \frac{1}{N}\sum_{i=1}^{N} f_{1,i}(\hat x') )  \\
    =&\lambda_1\lambda'_2\sum_{i=1}^{N} f_{1,i}(\hat x') +   \lambda_2\lambda'_2\sum_{i=N+1}^{N_s} f_{2,i}(\hat x') \\
    =& \lambda'_2 f(\hat x';\lambda_1,\lambda_2),
\end{align}
which contradicts the definition of $\hat x$ in~\eqnref{x1_def} (i.e., $\hat{x}$ maximizes $f(x;\lambda_1,\lambda_2)$), completing the proof. 
\end{proof}

\begin{lemma}\label{lem:cmp_free_lemma}
Suppose we have two series of functions $\{f_{1,i}(x)\}$ and $\{f_{2,i}(x)\}$,  with two series of non-negative weighting parameters $\lamm_1 = [\lambda_{1,i}]_{i=1}^{N},\lamm_2=[\lambda_{2,i}]^{N_s}_{i=N+1}$ satisfying $\sum_{i=1}^N\lambda_{1,i} + \sum_{i=N+1}^{N_s}\lambda_{2,i} = 1$. We define the final objective function $f(\cdot)$ as:
\begin{align}
    f(x;\lamm_1, \lamm_2) =  \sum_{i=1}^{N} \lambda_{1,i}f_{1,i}(x) + \sum_{i=N+1}^{N_s} \lambda_{2,i} f_{2,i}(x).
\end{align}
We assume two pairs of parameters $(\lamm_1, \lamm_2)$ and 
$(\lamm'_1, \lamm'_2)$, where
\begin{align}
   & \lambda_{1,i} \ge \lambda'_{1,i},~~i\in \{1,2,...,N\},  \label{eq:lambda_1s}\\
   &\lambda_{2,i} \le \lambda'_{2,i},~~i\in \{N+1,N+2,...,N_s\} . \label{eq:lambda_2s}
\end{align}

Defining the optimal values of the objective function for different weighting parameters as
\begin{align}
\label{x1_free_def}
    & \hat x = \arg \max_x f(x;\lamm_1,\lamm_2), \\
\label{x2_free_def}
    & \hat x' = \arg \max_x f(x;\lamm'_1,\lamm'_2),\\
\label{x_best}
    & x^* = \arg\max f(x,\frac{\1}{N},\0).
\end{align}
Under the following \textbf{Assumptions} (with $\1$ and $\0$ denoting vectors with all entries equal to $1$ and $0$, respectively):
\begin{enumerate}
    \item  $f(\hat x,\0,\lamm_2)\le  f(\hat x',\0,\lamm_2)$.
    \item $f(x;\lamm,\0) \ge f(x';\lamm,\0)$, iff $\Vert x-x^*\Vert\le \Vert x'-x^*\Vert,~~\lamm\ge 0,~~\Vert\lamm\Vert_1 = 1$.
\end{enumerate}

we have that
\begin{align}
     f(\hat x;\frac{\1}{N},\0) \ge f(\hat x';\frac{\1}{N},\0) .
\end{align}
\end{lemma}
   
\begin{proof}
We start with proving the following equality by contradiction:
\begin{align} \label{dist_x}
    \Vert \hat x- x^*\Vert\le \Vert \hat x'-x^*\Vert.
\end{align}
Specifically, if
\begin{align}
    \Vert \hat x- x^*\Vert > \Vert \hat x'-x^*\Vert,
\end{align}
leveraging the Assumption 1 and 2 above, we have that
\begin{align}
    f(\hat x;\lamm_1,\lamm_2) = f(\hat x;\lamm_1,\0) + f(\hat x;\0,\lamm_2)
    < f(\hat x';\lamm_1,\0) + f(\hat x';\0,\lamm_2) = f(\hat x';\lamm_1,\lamm_2),
\end{align}
which contradicts~\eqnref{x1_free_def}. Therefore, ~\eqnref{dist_x} holds.

Combining~\eqnref{dist_x} and Assumption 2 above, we have that
\begin{align}
    f(\hat x;\frac{\1}{N},\0) \ge f(\hat x';\frac{\1}{N},\0),
\end{align}
concluding the proof. 
\end{proof}




Based on the definitions and lemmas above, we have the following theorems: 
\begin{theorem}[\textbf{Advantage of $\Tha_A$ in the Simplified Case}]
\label{thm:advantage_a}
With~\defref{def:three_weights} and~\defref{def:opt_para}, comparing $\Tha_I$, $\Tha_G$, and $\Tha_A$ by evaluating them on the marginal log-likelihood of non-stop-words, i.e., $\mathcal{L}(\cdot,w\rightarrow \mbox{GT})$, we have that
\begin{align} \label{att_sup_gmm1}
    \mathcal{L}_{\text{GMM}} (\Tha_I; \w\to \mbox{GT}) \le \mathcal{L} _{\text{GMM}}(\Tha_A; \w\to \mbox{GT})\le \mathcal{L} _{\text{GMM}}(\Tha_G; \w\to \mbox{GT}).
\end{align}
\end{theorem}
\begin{proof}
First, by definition one can easily find that $\Tha_G$ achieves the largest $\mathcal{L}(\cdot;\w\rightarrow \mbox{GT})$ among the three:
\begingroup\makeatletter\def\f@size{6.5}\check@mathfonts
\begin{align}
    \max[\mathcal{L}_{GMM} (\Tha_I; \w\to \mbox{GT}),\mathcal{L} _{GMM}(\Tha_A; \w\to \mbox{GT})]\le \max_{\Tha} \mathcal{L}_{GMM} (\Tha; \w\to \mbox{GT})=\mathcal{L}_{GMM} (\Tha_G; \w\to \mbox{GT}).\label{eq:iag}
\end{align}
\endgroup

Next, we set $\{w_i\}_{i=1}^N$ to $\lambda_1$ and $\{w_i\}_{i=N+1}^{N+N_s}$ to $\lambda_2$, respectively; we rewrite $\log\sum_{k=1}^{K}  \pi_k\mathcal{N}(\e_i;\muu_k,\Si_k)$ as $f_{1,i}(x)$ for $i\in \{1,2,\dots,N\}$ and $f_{2,i}(x)$ for $i\in \{N+1,N+1,\dots,N+N_s\}$, where $x$ corresponds to $\Tha\triangleq (\pii,\{\muu_k, \Si_k\}_{k=1}^K)$. By~\lemref{lem:cmp_lemma}, we have that
\begin{align}
\mathcal{L}_{\text{GMM}} (\Tha_A; \w\to \mbox{GT})\le \mathcal{L} _{\text{GMM}}(\Tha_G; \w\to \mbox{GT}).\label{eq:ag}
\end{align}


Combining \eqnref{eq:iag} and \eqnref{eq:ag} concludes the proof. 
\end{proof}

\thmref{thm:advantage_a} shows that under mild assumptions, the attention-based weights can help produce better estimates of $\Tha$ in the presence of noisy stop-words and therefore learns higher-quality topics from the corpus, improving interpretability of FLMs. 

\begin{theorem}[\textbf{Advantage of $\Tha_A$ in the General Case}]
\label{thm:advantage_a_general}
With~\defref{def:three_free_weights} and~\defref{def:opt_para}, comparing $\Tha_I$, $\Tha_G$, and $\Tha_A$ by evaluating them on the marginal log-likelihood of non-stop-words, i.e., $\mathcal{L}_{GMM}(\cdot,w\rightarrow \mbox{GT})$, we have that
\begin{align} \label{att_sup_gmm2}
    \mathcal{L}_{\text{GMM}} (\Tha_I; \w\to \mbox{GT}) \le \mathcal{L}_{\text{GMM}} (\Tha_A; \w\to \mbox{GT})\le \mathcal{L}_{\text{GMM}} (\Tha_G; \w\to \mbox{GT}).
\end{align}
\end{theorem}
\begin{proof}
First, by definition one can easily find that $\Tha_G$ achieves the largest $\mathcal{L}(\cdot;\w\rightarrow \mbox{GT})$ among the three:
\begingroup\makeatletter\def\f@size{6.5}\check@mathfonts
\begin{align}
    \max[\mathcal{L}_{\text{GMM}} (\Tha_I; \w\to \mbox{GT}),\mathcal{L}_{\text{GMM}} (\Tha_A; \w\to \mbox{GT})]\le \max_{\Tha} \mathcal{L}_{\text{GMM}} (\Tha; \w\to \mbox{GT})=\mathcal{L}_{\text{GMM}} (\Tha_G; \w\to \mbox{GT}).\label{eq:iag2_gmm}
\end{align}
\endgroup
Next, we invoke~\lemref{lem:cmp_free_lemma} by (1) setting $\{w_i\}_{i=1}^N$ to $\lamm_1$ and $\{w_i\}_{i=N+1}^{N+N_s}$ to $\lamm_2$, respectively, and (2) rewriting $\log\sum_{k=1}^{K}  \pi_k\mathcal{N}(\e_i;\muu_k,\Si_k)$ as $f_{1,i}(x)$ for $i\in \{1,2,\dots,N\}$ and $f_{2,i}(x)$ for $i\in \{N+1,N+1,\dots,N+N_s\}$, where $x$ corresponds to $\Tha\triangleq (\pii,\{\muu_k, \Si_k\}_{k=1}^K)$. By~\lemref{lem:cmp_free_lemma}, we then have that
\begin{align}
\mathcal{L}_{\text{GMM}} (\Tha_A; \w\to \mbox{GT})\le \mathcal{L}_{\text{GMM}} (\Tha_G; \w\to \mbox{GT}).\label{eq:ag2_gmm}
\end{align}
Note that because $f_{1,i}(\cdot)$ and $f_{2,i}(\cdot)$ are Gaussian, therefore Assumption 1 and 2 in~\lemref{lem:cmp_free_lemma} hold naturally under mild regularity conditions.

Combining \eqnref{eq:iag2_gmm} and \eqnref{eq:ag2_gmm} concludes the proof. 
\end{proof}

\subsection{VALC as Interpreters}

As mentioned in~\eqnref{sec:app_elbo_expansion} , the ELBO of the marginal likelihood (denoting as $\Tha$ the collection of parameters $\ph,\gamm$ and $\{\muu_k, \Si_k\}_{k=1}^K$) is as follows: 
\begingroup\makeatletter\def\f@size{6.5}\check@mathfonts
\begin{align}\label{eq:elbo-tha}
    \mathcal{L}_{\text{VALC}}(\Tha;\{w_i\}) &= \sum\nolimits_{j=1}^{L'}\EB_q[\log p(\e_{mj}|z_{mj},\muu_{z_{mj}}, \Si_{z_{mj}})]\nonumber\\
    &=\sum_{m,j} w_{mj}\sum_{k}\phi_{mjk}  \{-\frac{1}{2}(\e_{mj}-\muu_k)^T\Si_k^{-1}(\e_{mj}-\muu_k)    -\log[(2\pi)^{H/2} \vert \Si_k\vert^{1/2}]\}.
\end{align}
\endgroup

Based on the definitions and lemmas above, we have the following theorems: 
\begin{theorem}[\textbf{Advantage of $\Tha_A$ in the Simplified Case}]
\label{thm:advantage_valance_simple}
With~\defref{def:three_weights} and~\defref{def:opt_para}, comparing $\Tha_I$, $\Tha_G$, and $\Tha_A$ by evaluating them on the marginal log-likelihood of non-stop-words, i.e., $\mathcal{L}(\cdot,w\rightarrow \mbox{GT})$, we have that
\begin{align} \label{att_sup_valance1}
    \mathcal{L}_{\text{VALC}} (\Tha_I; \w\to \mbox{GT}) \le \mathcal{L}_{\text{VALC}} (\Tha_A; \w\to \mbox{GT})\le \mathcal{L}_{\text{VALC}} (\Tha_G; \w\to \mbox{GT}).
\end{align}
\end{theorem}
\begin{proof}
First, by definition one can easily find that $\Tha_G$ achieves the largest $\mathcal{L}(\cdot;\w\rightarrow \mbox{GT})$ among the three:
\begingroup\makeatletter\def\f@size{6.5}\check@mathfonts
\begin{align}
    \max[\mathcal{L}_{\text{VALC}} (\Tha_I; \w\to \mbox{GT}),\mathcal{L}_{\text{VALC}} (\Tha_A; \w\to \mbox{GT})]\le \max_{\Tha} \mathcal{L}_{\text{VALC}} (\Tha; \w\to \mbox{GT})=\mathcal{L}_{\text{VALC}} (\Tha_G; \w\to \mbox{GT}).\label{eq:iag-valance}
\end{align}
\endgroup

Next, we set $\cup_m \{w_{mj}\}_{j=1}^{N_m}$ to $\lambda_1$ and $\cup_m \{w_{mj}\}_{j=N_m+1}^{N_m+N_{m,s}}$ to $\lambda_2$, respectively; we rewrite $\sum_{i}\phi_{mji}  \{-\frac{1}{2}(\e_{mj}-\muu_i)^T\Si_i^{-1}(\e_{mj}-\muu_i)    -\log[(2\pi)^{d/2} \vert \Si_i\vert^{1/2}]\}$ as $f_{1,j}(x)$ for $j\in \cup_m \{1,2,\dots,N_m\}$ and $f_{2,j}(x)$ for $j\in \cup_m \{N_m+1,N_m+1,\dots,N_m+N_{m,s}\}$, where $x$ corresponds to $\Tha\triangleq (\ph,\gamm,\{\muu_k, \Si_k\}_{k=1}^K)$. By~\lemref{lem:cmp_lemma}, we have that
\begin{align}
\mathcal{L}_{\text{VALC}} (\Tha_A; \w\to \mbox{GT})\le \mathcal{L}_{\text{VALC}} (\Tha_G; \w\to \mbox{GT}).\label{eq:ag-valance}
\end{align}


Combining \eqnref{eq:iag-valance} and \eqnref{eq:ag-valance} concludes the proof. 
\end{proof}

\thmref{thm:advantage_valance_simple} shows that under mild assumptions, the attention-based weights can help produce better estimates of $\Tha$ in the presence of noisy stop-words and therefore learns higher-quality topics from the corpus, improving and interpretability of FLMs. 

\begin{theorem}[\textbf{Advantage of $\Tha_A$ in the General Case}]
\label{thm:advantage_valance_general}
With~\defref{def:three_free_weights} and~\defref{def:opt_para}, comparing $\Tha_I$, $\Tha_G$, and $\Tha_A$ by evaluating them on the marginal log-likelihood of non-stop-words, i.e., $\mathcal{L}_{VALC}(\cdot,w\rightarrow \mbox{GT})$, we have that
\begin{align} \label{att_sup_valance2}
    \mathcal{L}_{\text{VALC}} (\Tha_I; \w\to \mbox{GT}) \le \mathcal{L}_{\text{VALC}} (\Tha_A; \w\to \mbox{GT})\le \mathcal{L}_{\text{VALC}} (\Tha_G; \w\to \mbox{GT}).
\end{align}
\end{theorem}
\begin{proof}
First, by definition one can easily find that $\Tha_G$ achieves the largest $\mathcal{L}(\cdot;\w\rightarrow \mbox{GT})$ among the three:
\begingroup\makeatletter\def\f@size{6.5}\check@mathfonts
\begin{align}
    \max[\mathcal{L}_{\text{VALC}} (\Tha_I; \w\to \mbox{GT}),\mathcal{L}_{\text{VALC}}(\Tha_A; \w\to \mbox{GT})]\le \max_{\Tha} \mathcal{L}_{\text{VALC}} (\Tha; \w\to \mbox{GT})=\mathcal{L}_{\text{VALC}} (\Tha_G; \w\to \mbox{GT}).\label{eq:iag2_valance}
\end{align}
\endgroup
Next, we invoke~\lemref{lem:cmp_free_lemma} by (1) setting $\cup_m \{w_{mj}\}_{j=1}^{N_m}$ to $\lamm_1$ and $\cup_m \{w_{mj}\}_{j=N_m+1}^{N_m+N_{m,s}}$ to $\lamm_2$, respectively, and (2) rewriting $\sum_{i}\phi_{mji}  \{-\frac{1}{2}(\e_{mj}-\muu_i)^T\Si_i^{-1}(\e_{mj}-\muu_i)    -\log[(2\pi)^{d/2} \vert \Si_i\vert^{1/2}]\}$ as $f_{1,j}(x)$ for $j\in \cup_m \{1,2,\dots,N_m\}$ and $f_{2,j}(x)$ for $j\in \cup_m \{N_m+1,N_m+1,\dots,N_m+N_{m,s}\}$, where $x$ corresponds to $\Tha\triangleq (\ph,\gamm,\{\muu_k, \Si_k\}_{k=1}^K)$. By~\lemref{lem:cmp_free_lemma}, we then have that
\begin{align}
\mathcal{L} _{\text{VALC}}(\Tha_A; \w\to \mbox{GT})\le \mathcal{L}_{\text{VALC}} (\Tha_G; \w\to \mbox{GT}).\label{eq:ag2_valance}
\end{align}
Note that because $f_{1,j}(\cdot)$ and $f_{2,j}(\cdot)$ are very close to Gaussian, therefore Assumption 1 and 2 in~\lemref{lem:cmp_free_lemma} hold naturally under mild regularity conditions.

Combining \eqnref{eq:iag2_valance} and \eqnref{eq:ag2_valance} concludes the proof. 
\end{proof}

\end{document}